\documentclass[runningheads]{llncs}

\usepackage{eccv}

\usepackage{eccvabbrv}

\usepackage{graphicx}
\usepackage{booktabs}

\usepackage[accsupp]{axessibility}  

\usepackage{hyperref}

\usepackage{orcidlink}

\usepackage{amsmath,amsfonts,bm}

\def\eqref#1{equation~\ref{#1}}

\def\1{\bm{1}}
\newcommand{\train}{\mathcal{D}}

\DeclareMathAlphabet{\mathsfit}{\encodingdefault}{\sfdefault}{m}{sl}
\SetMathAlphabet{\mathsfit}{bold}{\encodingdefault}{\sfdefault}{bx}{n}

\definecolor{LightCyan}{rgb}{0.78,0.94,1}
\definecolor{Gray}{gray}{0.85}
\definecolor{lightgreen}{rgb}{0.286,0.639,0.345}
\definecolor{verylightgray}{gray}{0.95}
\definecolor{lightred}{rgb}{1.0,0.7,0.7}
\definecolor{darkred}{rgb}{0.698, 0.133, 0.133}

\usepackage{wrapfig}
\usepackage{tcolorbox}
\usepackage{pifont}
\usepackage{color, colortbl}
\usepackage{subcaption}

\usepackage{multirow}
\usepackage{makecell}
\usepackage{lipsum}
\usepackage[utf8]{inputenc}
\usepackage{xcolor}
\usepackage{ulem}

\usepackage{graphicx}  
\usepackage{caption}   
\usepackage{float}     
\usepackage{array}     

\usepackage{mathrsfs}
\usepackage{mathbbol}
\usepackage{bbm}
\usepackage{algorithm}
\usepackage{algorithmic}

\definecolor{mybd}{HTML}{b22222}
\definecolor{mybg}{HTML}{cccccc}
\newtcbox{\bottomhl}{
    on line,            
    boxsep=0pt,         
    left=0pt,right=0pt, 
    top=1pt,bottom=1pt, 
    colframe=mybd,        
    colback=white,      
    boxrule=0.8pt,
    arc=0pt             
}
\usepackage{dsfont}
\usepackage{placeins}

\newtheorem{mythr}{\bf{Theorem}}
\newtheorem{mylemma}{\bf{Lemma}}
\newtheorem{myremark}{\bf{Remark}}
\newtheorem{mydef}{\bf{Definition}}

\begin{document}

\title{On the Plasticity Collapse\\ in Continual Machine Unlearning} 

\titlerunning{Plasticity Collapse in Continual Machine Unlearning}

\author{Yingdan Shi\inst{1} \and
Xiang Xu\inst{2} \and
Kaize Ding\inst{3} \and
Alfred O. Hero\inst{4} \and
Ren Wang\inst{1}\thanks{Corresponding author: rwang74@illinoistech.edu}}

\authorrunning{Y. Shi et al.}

\institute{Illinois Institute of Technology, Chicago IL 60616, USA\and
Amazon, New York NY 10001, USA\and
Northwestern University, Evanston IL 60208, USA\and
University of Michigan, Ann Arbor MI 48109, USA}

\maketitle

\begin{abstract}
Machine unlearning enables deep neural networks to selectively remove the influence of specific data in response to privacy and regulatory requirements. While prior work largely studies single-shot unlearning, real-world systems must accommodate continual unlearning, where multiple unlearning requests occur sequentially over time. In this work, we identify a fundamental limitation of this setting: plasticity collapse, a progressive breakdown in a model's ability to effectively forget. Through theoretical analysis of continual unlearning dynamics, we show that continual unlearning operations accumulate geometric constraints in parameter space, leading to saturated subspaces that restrict future updates. This structural effect induces two distinct failure modes: (1) Forward failure -- diminishing forgetting quality for subsequent tasks, and (2) Backward failure -- spontaneous re-memorization of previously forgotten information. Extensive experiments across multiple architectures, datasets, and methods in image classification confirm that plasticity collapse is not an artifact of specific implementations, but a pervasive phenomenon inherent to continual unlearning. Our findings reveal a critical barrier to the long-term reliability of machine unlearning systems and motivate the development of plasticity-preserving unlearning algorithms. Our code is available at \url{https://github.com/TIML-Group/Continual-Machine-Unlearning-Plasticity-Collapse}.


\end{abstract}

\section{Introduction}
\label{sec:introduction}

The rapid proliferation of machine learning across sensitive domains has intensified scrutiny over data governance and user privacy. Recent legislative frameworks, including the European Union's General Data Protection Regulation (GDPR) and the California Consumer Privacy Act (CCPA), establish the ``right to be forgotten'', a legal mandate requiring organizations to delete user data upon request. Beyond regulatory compliance, there is an escalating need to purge harmful or biased information from models to mitigate potential misuses, such as the generation of toxic content or the propagation of social biases. For machine learning models, compliance extends beyond removing records from training sets. It requires eliminating information that has been encoded into model parameters during training, a process known as machine unlearning.

Significant progress has been made in developing efficient unlearning methods for single unlearning requests. However, practical deployments face a more complex reality: models must accommodate multiple sequential unlearning\footnote{In this work, the terms \textit{continual unlearning} and \textit{sequential unlearning} are used interchangeably.} requests over their operational lifetime. A content moderation system, for instance, may need to progressively forget different categories of prohibited content as policies evolve. A recommendation system must continuously honor unlearning requests without complete retraining. These scenarios demand that unlearning methods remain effective not just once, but repeatedly and reliably over time.

\begin{figure}
    \vspace{-6mm}
    \centering
    \includegraphics[width=0.9\linewidth]{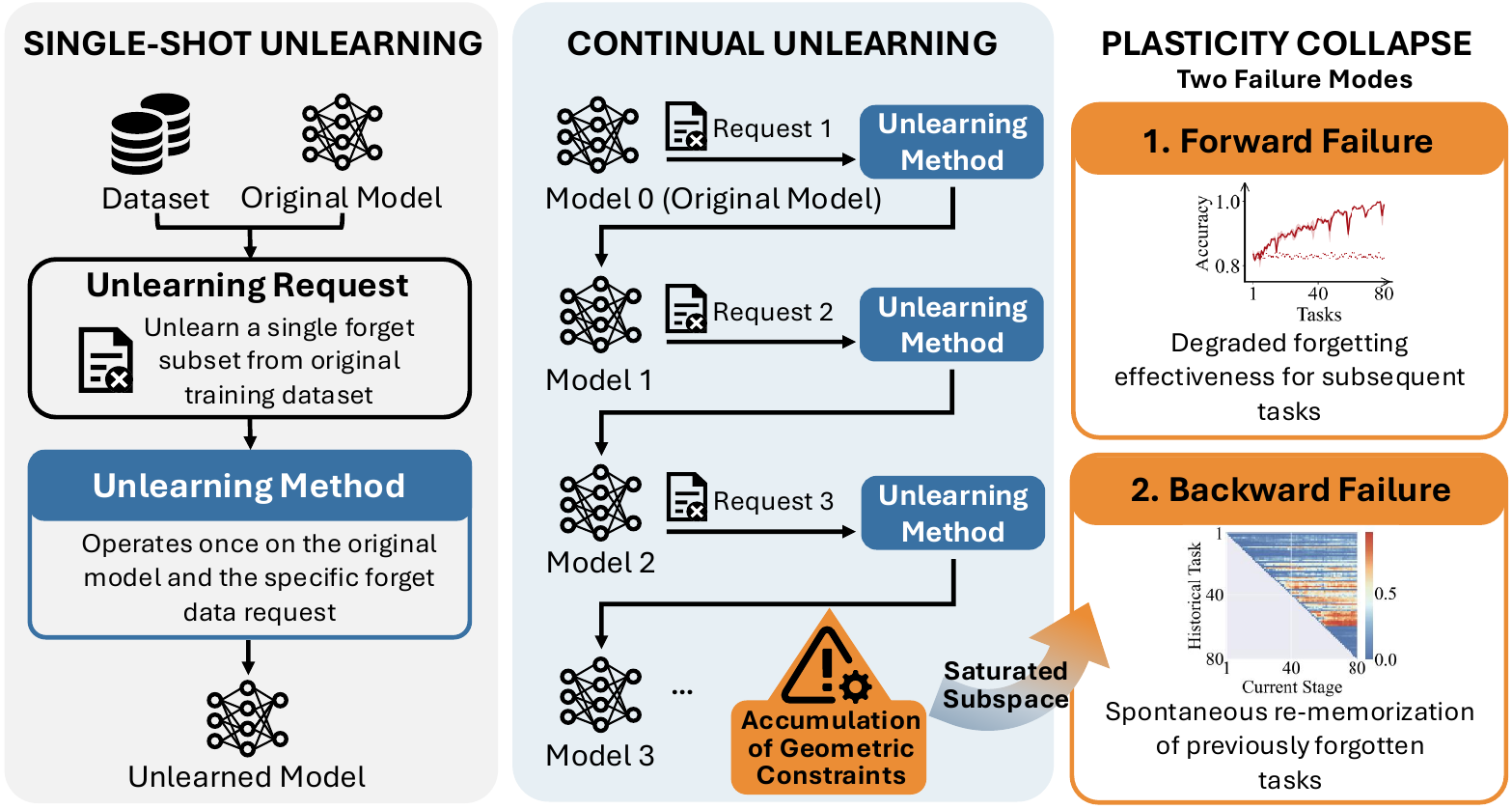}
    \vspace{-2mm}
    \caption{Comparison between single-shot unlearning and continual unlearning. Unlike the single-shot setting, continual unlearning suffers from the accumulation of geometric constraints, leading to a saturated subspace that induces plasticity collapse. This collapse is characterized by two distinct phenomena: forward failure, which represents a progressive decline in forgetting quality, and backward failure, referring to the spontaneous re-memorization of previously forgotten tasks.}
    \label{fig:overview}
    \vspace{-6mm}
\end{figure}
The overview of this work is shown in Fig.~\ref{fig:overview}. In this work, we identify a fundamental limitation of existing unlearning methods in continual settings: models progressively lose their ability to effectively forget, a phenomenon we term \textbf{\textit{plasticity collapse}}. Through extensive experiments, we identify two distinct failure modes induced by plasticity collapse: (1) \textbf{Forward failure}, in which forgetting quality deteriorates progressively for subsequent tasks, and (2) \textbf{Backward failure}, in which previously forgotten tasks spontaneously re-emerge. Through theoretical analysis of continual unlearning dynamics, we show that successive unlearning operations accumulate geometric constraints in parameter space, leading to saturated subspaces that restrict future parameter updates. This structural degradation manifests in both failure modes. These findings raise serious concerns about the long-term reliability of machine unlearning in the continual unlearning setting. If forgetting quality degrades over sequential requests or earlier forgetting is silently reversed, the privacy guarantees that unlearning methods promise become difficult to sustain. Our contributions are as follows.
\begin{itemize}
    \item We are the first to systematically characterize plasticity collapse in the continual unlearning. Through extensive experiments spanning diverse architectures, datasets, and unlearning methods in image classification, we demonstrate that this is not an artifact of specific implementations, but a pervasive phenomenon inherent to the continual unlearning setting.

    \item We provide a theoretical analysis of the underlying mechanism. Our analysis shows that the accumulation of parameter constraints from successive unlearning operations creates increasingly restrictive subspaces, leading to degraded optimization and persistent residual information.

    \item We design targeted experiments to isolate and validate the mechanisms underlying plasticity collapse, bridging our theoretical predictions with observed empirical behavior.
\end{itemize}

While our analysis focuses on image classification tasks, the identified mechanisms of subspace saturation and geometric constraints stem from the fundamental optimization dynamics of neural networks. Consequently, this work lays a critical theoretical and diagnostic foundation for understanding plasticity collapse in broader contexts, such as generative modeling and large language models (LLMs), where sequential unlearning is equally pivotal yet more complex.

\section{Related Work}
\label{sec:related_work}

\subsection{Machine Unlearning}

Machine unlearning seeks to remove the influence of specific training data from a learned model without requiring complete retraining. Early approaches framed unlearning as the inverse of learning, applying gradient ascent on the data to be forgotten~\cite{cao2015towards,ginart2019making}. However, naive gradient ascent often causes catastrophic degradation of model utility on retained data.

Exact unlearning methods offer certified guarantees by partitioning training data and maintaining sub-models that can be efficiently retrained upon deletion~\cite{guo2019certified,neel2021descent,bourtoule2021machine}. SISA training~\cite{bourtoule2021machine}, for example, divides the dataset into shards and trains separate models on each, allowing selective retraining when data is removed. While these approaches provide strong guarantees, they incur substantial storage and computational overhead, limiting their scalability, particularly in continual unlearning settings.

Approximate unlearning methods~\cite{salun2024,Amnesiac2021,neggrad,ga2022,finetune2021,shi2025mcu,shi2025redefining
} trade formal guarantees for efficiency, offering practical alternatives through diverse techniques including fine-tuning, parameter perturbation, and selective gradient manipulation. These methods are generally more applicable to real-world deployment but have been studied almost exclusively in single-shot unlearning settings.

\subsection{Continual Unlearning}

A line of work addresses scenarios where a model alternates between continual learning and selective forgetting~\cite{chatterjee2024unified,tang2025acu,adhikari2025unlearning,huang2025unified}. In this mixed setting, the model continues to receive new training data alongside unlearning requests, which helps preserve model utility throughout the process.

More closely related to our work, \cite{feng2025fg} studied incremental unlearning, where the goal is to achieve thorough forgetting across sequential unlearning tasks on a fixed model, without interleaving new learning. That is, the model must completely remove targeted knowledge while retaining the rest, purely through successive forgetting operations. This purely sequential unlearning setting is the focus of our work. We go further by identifying a fundamental limitation of this setting, namely plasticity collapse, and providing both theoretical and empirical characterization of this failure mode.

\subsection{Neural Network Plasticity}

Loss of plasticity has been identified as a fundamental challenge in deep learning, referring to the progressive decline in a network's capacity to learn new information over time~\cite{kumar2023maintaining,dohare2024loss,lyle2022understanding}. \cite{lyle2022understanding} showed that neural networks in deep reinforcement learning lose learning capacity even when the underlying task remains stationary. Proposed mechanisms include dormant neurons~\cite{sokar2023dormant}, implicit rank collapse~\cite{kumar2020implicit}, and gradient interference~\cite{goyal2022inductive}. \cite{kumar2023maintaining} further demonstrated that continual learning agents suffer similar rigidity, attributing the effect to feature reuse constraints and conflicting gradient signals. Remedies include weight resetting~\cite{dohare2024loss}, optimization adaptations~\cite{mirzadeh2022wide}, and the shrink-and-perturb strategy~\cite{ash2020warm,evron2022catastrophic}, which combines parameter pruning with controlled perturbation.

Despite this progress, plasticity in the context of machine unlearning remains unexplored. The mechanisms governing learning plasticity do not transfer directly to unlearning, since the optimization objectives and parameter-space constraints differ fundamentally. Our work is the first to characterize plasticity collapse specifically within the unlearning context, establishing it as a distinct and pervasive phenomenon with its own theoretical grounding.

\section{Plasticity Collapse in Continual Unlearning}
\label{sec:method}

\subsection{Preliminaries and Notation}
\label{subsec:prelim}

\paragraph{Continual unlearning setup.}
Let $\mathcal{D} = \{(x_i, y_i)\}_{i=1}^N$ denote the full training dataset used to train an initial model with parameters $\bm{\theta}_0 \in \mathbb{R}^d$.
We consider a sequence of $T$ unlearning requests arriving one at a time.
At each stage $t \in \{1, \dots, T\}$, a subset is designated as the forget set $F_t \subset \mathcal{D}$, containing samples whose influence must be removed from the model. An optional retain set $R_t \subseteq \mathcal{D} \setminus F_t$ is provided to preserve model utility on non-forgotten data. After processing request $t$, the model parameters are updated from $\bm{\theta}_{t-1}$ to $\bm{\theta}_t$.

\paragraph{Unlearning objective.}
For any dataset $S = \{(x_i, y_i)\}_{i=1}^{n_S}$, let $\ell(\bm{\theta}; x, y)$ denote the per-sample loss (e.g., cross-entropy). We define the empirical loss over $S$ as
\begin{equation}
    \mathcal{L}_S(\bm{\theta}) \;:=\; \frac{1}{n_S} \sum_{i=1}^{n_S} \ell(\bm{\theta};\, x_i,\, y_i).
    \label{eq:empirical_loss}
\end{equation}
Most unlearning methods share a common two-component structure: a \textbf{forget loss} $\mathcal{L}_{F_t}(\bm{\theta})$ that drives the model away from predictions on $F_t$, and a \textbf{retain loss} $\mathcal{L}_{R_t}(\bm{\theta})$ that anchors model behavior on $R_t$. Methods differ primarily in how each loss function is designed. For example, NegGrad+~\cite{neggrad} additionally minimizes the retain loss,  RandomLabeling~\cite{Amnesiac2021} replaces ground-truth labels in $F_t$ with random ones, and SalUn~\cite{salun2024} further restricts updates to salient parameters. Despite these differences, in the local geometry of parameter space, the parameter update induced by any such method has a positive projection onto the gradient ascent direction of $\mathcal{L}_{F_t}$, reflecting a shared inductive bias toward reducing the model's fit on $F_t$. Thus, all such methods fit within the unified objective
\begin{equation}
    \mathcal{J}_t(\bm{\theta})
    \;:=\;
    \lambda\,\mathcal{L}_{R_t}(\bm{\theta})
    \;-\;
    \mathcal{L}_{F_t}(\bm{\theta}),
    \qquad \lambda \ge 0,
    \label{eq:unlearn_obj}
\end{equation}
where $\lambda$ controls the balance between forgetting and utility preservation. One step of gradient descent on $\mathcal{J}_t$ gives
\begin{equation}
    \bm{\theta}_t
    \;=\;
    \bm{\theta}_{t-1} - \eta_t \nabla_{\bm{\theta}} \mathcal{J}_t(\bm{\theta}_{t-1})
    \;=\;
    \bm{\theta}_{t-1}
    - \eta_t \Bigl(\lambda\,\nabla\mathcal{L}_{R_t}(\bm{\theta}_{t-1}) - \nabla\mathcal{L}_{F_t}(\bm{\theta}_{t-1})\Bigr),
    \label{eq:gd_unlearn}
\end{equation}
where $\eta_t > 0$ is the step size at stage $t$. {For tractability, we analyze the single-step case in Eq.~\ref{eq:gd_unlearn}, but extending to multiple steps per task does not affect our theoretical conclusions.} Our theoretical analysis is developed within this general framework and does not presuppose any particular design of $\mathcal{L}_{F_t}$ or $\mathcal{L}_{R_t}$. The plasticity collapse we characterize is therefore a structural consequence of sequentially applying any training-based unlearning method.

\subsection{Definition of Plasticity Collapse}
\label{subsec:definition}

Through extensive experiments across diverse datasets and architectures, we identify two distinct failure modes inherent in the continual unlearning setting. Prior to our theoretical analysis, we formally define these modes, which we attribute to the plasticity collapse that occurs during sequential unlearning.

\begin{mydef}[Plasticity collapse]
\label{def:plasticity_collapse}
Let $\{\bm{\theta}_t\}_{t=1}^T$ be the sequence of model parameters produced by an unlearning method applied to a sequence of forget sets $\{F_t\}_{t=1}^T$. Assuming the model maintains a stable level of predictive utility during the entirety of the unlearning stages, we say the model exhibits plasticity collapse if, as the number of unlearning tasks $t$ increases, one or both of the following failure modes manifest:
\begin{enumerate}
    \item[\ding{202}] \textbf{Forward Failure.}     The forgetting quality on task $t$ progressively deteriorates as the number of sequential tasks increases.
    \item[\ding{203}] \textbf{Backward Failure.} Information from an earlier forget data $F_s$ ($s < t$) re-emerges in the model at a later stage $t > s$ without any explicit re-exposure to the data in $F_s$, which is a spontaneous re-memorization phenomenon.
\end{enumerate}
\end{mydef}

Both failure modes are empirically demonstrated in Section~\ref{sec:experiment}. The following theoretical analysis reveals that both failure modes stem from a single underlying cause: the accumulation of geometric constraints in parameter space across sequential unlearning tasks.

\subsection{Theoretical Analysis}
\label{subsec:theory}

In this section, we theoretically analyze why unlearning plasticity collapse exists in the continual unlearning dynamics. The key insight is that sequential unlearning tasks often induce updates that share directions in the parameter space $\mathbb{R}^d$. When the corresponding update operators are composed across tasks, repeated action along these shared directions can lead to geometric amplification under the conditions of the proposed theorems below. To prevent unbounded growth of the parameter iterates, the optimization dynamics must neutralize this expansion. This can occur either by suppressing future updates along the shared subspace, which limits the ability to modify predictions for new tasks (forward failure), or by introducing cancelling updates along the same directions, which partially undo earlier forgetting steps (backward failure).

\paragraph{Linear model.}
To analyze the continual unlearning dynamics in a tractable setting, we first instantiate the objective Eq.~\ref{eq:unlearn_obj} with squared loss on a linear model.
For any dataset $D = (\mathbf{X}_D,\, \mathbf{y}_D)$ with feature matrix $\mathbf{X}_D \in \mathbb{R}^{n_D \times d}$ and labels $\mathbf{y}_D \in \mathbb{R}^{n_D}$, the squared loss is
\begin{equation}
\vspace{-2mm}
    \mathcal{L}_D(\bm{\theta}) \;:=\; \frac{1}{2n_D}\|\mathbf{X}_D\bm{\theta} - \mathbf{y}_D\|_2^2,
    \label{eq:squared_loss}
\end{equation}
with gradient $\nabla\mathcal{L}_D(\bm{\theta}) = \frac{1}{n_D}\mathbf{X}_D^\top(\mathbf{X}_D\bm{\theta} - \mathbf{y}_D)$.
Substituting into Eq.~\ref{eq:gd_unlearn} yields an affine recursion in $\bm{\theta}$:
\begin{equation}
\vspace{-2mm}
    \bm{\theta}_t
    \;=\;
    \underbrace{\Bigl(\mathbf{I} + \mathbf{M}_t^{F} - \lambda \mathbf{M}_t^{R}\Bigr)}_{=:\,\mathbf{A}_t}\,\bm{\theta}_{t-1}
    \;-\;
    \underbrace{\Bigl(\mathbf{b}_t^{F} - \lambda \mathbf{b}_t^{R}\Bigr)}_{=:\,\mathbf{c}_t},
    \label{eq:affine_recursion}
\end{equation}
where the matrices $\mathbf{M}_t^F,\, \mathbf{M}_t^R \in \mathbb{R}^{d \times d}$ and vectors $\mathbf{b}_t^F,\, \mathbf{b}_t^R \in \mathbb{R}^d$ are defined as
\begin{equation}
    \mathbf{M}_t^{F} := \frac{\eta_t\mathbf{X}_{F_t}^\top \mathbf{X}_{F_t}}{n_{F_t}},
    \ \ 
    \mathbf{M}_t^{R} := \frac{\eta_t\mathbf{X}_{R_t}^\top \mathbf{X}_{R_t}}{n_{R_t}},
    \ \ 
    \mathbf{b}_t^{F} := \frac{\eta_t\mathbf{X}_{F_t}^\top \mathbf{y}_{F_t}}{n_{F_t}},
    \ \ 
    \mathbf{b}_t^{R} := \frac{\eta_t\mathbf{X}_{R_t}^\top \mathbf{y}_{R_t}}{n_{R_t}}.
    \label{eq:MF_MR_def}
\end{equation}
Note that $\mathbf{M}_t^F \succeq \mathbf{0}$ and $\mathbf{M}_t^R \succeq \mathbf{0}$ are positive semidefinite by construction.
To capture the directional alignment across sequential tracking tasks, let $W \subset \mathbb{R}^d$ denote the shared subspace spanned by the continuous model-drift trajectories.
The operator $A_t = I + M_t^F -\lambda M_t^R$ governs the evolution of parameter updates. In subspaces where the forget curvature dominates the scaled retain curvature, $A_t$ becomes expansive. The remainder of our analysis focuses on these expanding subspaces, which give rise to the operator-product instability underlying plasticity collapse.
For clarity, we assume a single gradient step per task in Eq.~\ref{eq:affine_recursion}. Performing $K$ steps simply replaces $\mathbf{A}_t$ with $\mathbf{A}_t^K$, which preserves the expansive structure on the subspace $W$ and leaves our theoretical conclusions below unchanged.

\paragraph{Special case: $\lambda = 0$.}
When no retain set is used, Eq.~\ref{eq:affine_recursion} simplifies to pure gradient ascent on the forget loss.
Writing $\mathbf{X}_t \in \mathbb{R}^{k \times d}$ and $\mathbf{y}_t \in \mathbb{R}^k$ for the forget data at stage $t$, the update becomes
\begin{equation}
    \bm{\theta}_{t+1} = (\mathbf{I} + \mathbf{M}_t)\,\bm{\theta}_t - \mathbf{b}_t,
    \qquad
    \mathbf{M}_t := \frac{\eta_t}{k}\mathbf{X}_t^\top \mathbf{X}_t \succeq \mathbf{0},
    \qquad
    \mathbf{b}_t := \frac{\eta_t}{k}\mathbf{X}_t^\top \mathbf{y}_t.
    \label{eq:ga_update}
\end{equation}
We define the update increment $\mathbf{u}_t := \bm{\theta}_{t} - \bm{\theta}_{t-1}$ and the forget span
\begin{equation}
    U_t \;:=\; \mathrm{Im}(\mathbf{X}_t^\top) \;=\; \mathrm{Im}(\mathbf{M}_t) \;\subseteq\; \mathbb{R}^d,
    \label{eq:forget_span}
\end{equation}
which is the subspace spanned by the input features of $F_t$.
Since $\nabla\mathcal{L}_t(\bm{\theta}) = \frac{1}{k}\mathbf{X}_t^\top(\mathbf{X}_t\bm{\theta} - \mathbf{y}_t) \in U_t$ for all $\bm{\theta}$, every gradient step satisfies $\mathbf{u}_t \in U_t$.
For a subspace $V \subseteq \mathbb{R}^d$, we write $\mathbf{P}_V \in \mathbb{R}^{d \times d}$ for the orthogonal projector onto $V$.

\begin{mydef}[Time-ordered product]
\label{def:top}
For indices $s \le t$, define the time-ordered product of the linear update operators as
\[
    \mathbf{A}_{t:s} \;:=\; \mathbf{A}_t \mathbf{A}_{t-1} \cdots \mathbf{A}_s,
\]
with the convention $\mathbf{A}_{t:t+1} := \mathbf{I}$. In the special case $\lambda = 0$, $\mathbf{A}_j = \mathbf{I} + \mathbf{M}_j$ and $\mathbf{A}_{t:s} = \prod_{j=s}^{t}(\mathbf{I} + \mathbf{M}_j)$.
\end{mydef}

\begin{mylemma}[Closed-form solution]
\label{lem:closed_form}
Let $\{\bm{\theta}_t\}$ follow the affine recursion Eq.~\ref{eq:affine_recursion}. Then for any $1 \le s \le t$,
\begin{equation}
    \bm{\theta}_t
    \;=\;
    \mathbf{A}_{t:s}\,\bm{\theta}_{s-1}
    \;-\;
    \sum_{j=s}^{t} \mathbf{A}_{t:j+1}\,\mathbf{c}_j.
    \label{eq:closed_form}
\end{equation}
\end{mylemma}

\begin{proof}
Substitute $\bm{\theta}_j = \mathbf{A}_j \bm{\theta}_{j-1} - \mathbf{c}_j$ recursively for $j = s, s{+}1, \dots, t$. Collecting terms gives Eq.~\ref{eq:closed_form}, using the convention $\mathbf{A}_{t:t+1} = \mathbf{I}$.
\end{proof}

Lemma~\ref{lem:closed_form} exposes the key structural tension: the term $\mathbf{A}_{t:s}\,\bm{\theta}_{s-1}$ grows with $t$ if $\mathbf{A}_{t:s}$ is expansive, and the forcing term $\sum_j \mathbf{A}_{t:j+1}\mathbf{c}_j$ must counterbalance this growth to keep $\bm{\theta}_t$ stable. The following theorems characterize when and how fast this growth occurs.

\begin{mythr}[Exponential growth on a shared subspace, $\lambda = 0$]
\label{thm:operator_growth}
Consider the $\lambda = 0$ dynamics Eq.~\ref{eq:ga_update}. Fix $s \le t$ and let $\mathcal{T} \subseteq \{s, \dots, t\}$ be any subset of task indices. Suppose there exists a nonzero subspace $W \subseteq \mathbb{R}^d$ satisfying:
\begin{enumerate}
    \item[\textbf{B1}] \emph{Shared subspace (invariance).} For every $j \in \mathcal{T}$, $\mathbf{M}_j W \subseteq W$.
    \item[\textbf{B2}] \emph{Uniform relevance.} There exists $\rho > 0$ such that for every $j \in \mathcal{T}$ and every $\mathbf{v} \in W$,
    \[
        \mathbf{v}^\top \mathbf{M}_j \mathbf{v} \;\ge\; \rho \|\mathbf{v}\|^2.
    \]
\end{enumerate}
Then for all $\mathbf{v} \in W$,
\begin{equation}
    \|\mathbf{A}_{t:s}\,\mathbf{v}\| \;\ge\; (1+\rho)^{|\mathcal{T}|}\,\|\mathbf{v}\|.
    \label{eq:op_growth}
\end{equation}
Consequently, for the trajectory $\theta_t$ to remain bounded, the dynamics must therefore suppress or counteract growth within $W$, which generically forces forward or backward failure.
\end{mythr}

The proof can be found in Appendix~\ref{app:proof_1}.

\begin{myremark}[Geometric interpretation of (B1) and (B2)]
\label{rem:conditions}
(\textbf{B1}) requires $W$ to be a subspace shared across unlearning tasks. In practice, $W$ may be approximated by the intersection of forget spans $\bigcap_t U_t$, or estimated as the dominant principal subspace of the stacked update matrix.
(\textbf{B2}) requires that forget features project nontrivially onto $W$: since $\mathbf{v}^\top \mathbf{M}_j \mathbf{v} = \frac{\eta_j}{k}\|\mathbf{X}_j \mathbf{v}\|^2$, (\textbf{B2}) is equivalent to $\|\mathbf{X}_j \mathbf{v}\|^2 \ge \frac{k\rho}{\eta_j}\|\mathbf{v}\|^2$, which holds when different forget sets repeatedly activate the same feature directions. {In practice, $\rho$ is small due to moderate curvature imbalance between forget and retain directions, combined with learning rate scaling.}
\end{myremark}

Theorem~\ref{thm:operator_growth} establishes the growth result for $\lambda = 0$. We now extend it to the general case with retain regularization.

\begin{mythr}[Exponential growth under forget-dominance, general $\lambda$]
\label{thm:expanding_mode_general}
Fix $s \le t$ and suppose there exists a nonzero subspace $W \subseteq \mathbb{R}^d$ such that for all $j \in \{s, \dots, t\}$:
\begin{enumerate}
    \item[\textbf{G1}] \emph{Invariance.} $\mathbf{A}_j W \subseteq W$, where $\mathbf{A}_j = \mathbf{I} + \mathbf{M}_j^F - \lambda \mathbf{M}_j^R$.
    \item[\textbf{G2}] \emph{Forget-dominance on $W$.} There exists $\rho > 0$ such that for all $\mathbf{v} \in W$,
    \[
        \mathbf{v}^\top (\mathbf{M}_j^F - \lambda \mathbf{M}_j^R) \mathbf{v} \;\ge\; \rho\|\mathbf{v}\|^2.
    \]
\end{enumerate}
Then for all $\mathbf{w} \in W$,
\begin{equation}
    \|\mathbf{A}_{t:s}\,\mathbf{w}\| \;\ge\; (1+\rho)^{t-s+1}\,\|\mathbf{w}\|.
    \label{eq:prod_growth_general}
\end{equation}
\end{mythr}

\begin{proof}
Under (\textbf{G2}), $\mathbf{A}_j|_W = (\mathbf{I} + \mathbf{M}_j^F - \lambda \mathbf{M}_j^R)|_W \succeq (1+\rho)\mathbf{I}_W$, so $\|\mathbf{A}_j \mathbf{w}\| \ge (1+\rho)\|\mathbf{w}\|$ for all $\mathbf{w} \in W$. Iterating via (\textbf{G1}) gives Eq.~\ref{eq:prod_growth_general}.
\end{proof}

\begin{myremark}[The role of $\lambda$ in (G2)]
\label{rem:G2_lambda}
(\textbf{G2}) holds when forget curvature exceeds the scaled retain curvature on $W$. When $\lambda = 0$ this reduces to (\textbf{B2}). For $\lambda > 0$, if the retain data overlaps strongly with $W$, then $\mathbf{M}_j^R|_W$ may dominate and $\mathbf{A}_j|_W$ becomes contractive, precluding exponential growth. However, when forget and retain sets are drawn from the same data distribution, their feature subspaces substantially overlap, and (\textbf{G2}) tends to hold for moderate $\lambda$.
\end{myremark}

\begin{myremark}[Connecting to the two failure modes]
\label{rem:G3_failuremodes}
Theorems~\ref{thm:operator_growth} and~\ref{thm:expanding_mode_general} characterize an operator-product instability: if there exists a (possibly approximate) shared subspace $W$ such that the linear operators $\mathbf{A}_t$ satisfy invariance on $W$ and expand on $W$ (i.e., $(\mathbf{M}_t^F - \lambda \mathbf{M}_t^R)|_W \succeq \rho \mathbf{I}_W$), then any component along $W$ is amplified exponentially under repeated composition, $\|\mathbf{A}_{t:s}\mathbf{w}\| \ge (1+\rho)^{t-s+1}\|\mathbf{w}\|$ for all $\mathbf{w} \in W$. Meanwhile, Lemma~\ref{lem:closed_form} shows that the iterate admits the decomposition $\bm{\theta}_t = \mathbf{A}_{t:s}\bm{\theta}_{s-1} - \sum_{j=s}^{t} \mathbf{A}_{t:j+1}\mathbf{c}_j$. Consequently, unless controlled, $\|\mathbf{P}_W \bm{\theta}_t\|$ grows geometrically, where $\mathbf{P}_W$ denotes the orthogonal projector onto $W$. For the trajectory to remain bounded, the dynamics must neutralize the expansion along $W$.
One stabilization mechanism is to progressively reduce motion along $W$, thereby limiting further excitation of the expanding directions. However, if successive forget tasks share components in $W$, this restriction directly reduces the model's ability to modify predictions for new forget requests, resulting in progressive deterioration of forgetting quality (\textbf{Forward Failure}). Alternatively, the second mechanism is cancellation within $W$, whereby later optimization steps generate updates that oppose the accumulated parameter displacement along the shared subspace. Specifically, Specifically, the propagated forcing terms $\sum_{j=s}^{t} \mathbf{A}_{t:j+1}\mathbf{c}_j$ collectively encode the historical forgetting trajectories produced by successive unlearning tasks within $W$. To maintain a bounded parameter trajectory despite the exponential expansion of the homogeneous component $\mathbf{A}_{t:s}\bm{\theta}_{s-1}$, the optimization dynamics must introduce opposing components inside the same shared subspace. Consequently, the net parameter displacement associated with earlier forgetting operations is partially reduced. Since previous forget tasks repeatedly modify the model through this common subspace, reducing the accumulated displacement along $W$ weakens the functional effect of earlier forgetting updates, allowing the model outputs on previously forgotten data to move back toward their pre-unlearning behavior. This manifests as spontaneous re-memorization of earlier forget tasks (\textbf{Backward Failure}). Overall, both failure modes thus emerge as distinct stabilization responses to the same operator-product expansion phenomenon characterized in Theorems~\ref{thm:operator_growth} and~\ref{thm:expanding_mode_general}.
\end{myremark}

Now consider a neural network $f(\mathbf{x}; \bm{\theta}): \mathbb{R}^p \to \mathbb{R}$ with parameters $\bm{\theta} \in \mathbb{R}^d$. Let $\bm{\theta}_0$ denote the initialization, and assume training remains in the Neural Tangent Kernel (NTK) regime, so that the first-order linearization
\vspace{-2mm}
\[
    f(\mathbf{x}; \bm{\theta}) \approx f(\mathbf{x}; \bm{\theta}_0) + \nabla_{\bm{\theta}} f(\mathbf{x}; \bm{\theta}_0)^\top (\bm{\theta} - \bm{\theta}_0)
\]
is valid throughout optimization. For each task $j$, define the NTK feature matrices $\mathbf{\Phi}_{F_j} \in \mathbb{R}^{n_{F_j} \times d}$ and $\mathbf{\Phi}_{R_j} \in \mathbb{R}^{n_{R_j} \times d}$, whose rows are $\nabla_{\bm{\theta}} f(\mathbf{x}_i; \bm{\theta}_0)^\top$ for samples in the forget and retain sets, respectively. Under squared loss with step size $\eta_j$, the parameter deviation $\bm{\delta}_j := \bm{\theta}_j - \bm{\theta}_0$ satisfies the affine recursion
\[
    \bm{\delta}_j = \mathbf{A}_j \bm{\delta}_{j-1} - (\mathbf{b}^F_j - \lambda \mathbf{b}^R_j),
\]
where $\mathbf{A}_j = \mathbf{I} + \mathbf{M}^F_j - \lambda \mathbf{M}^R_j, \quad
    \mathbf{M}^F_j = \frac{\eta_j}{n_{F_j}} \mathbf{\Phi}_{F_j}^\top \mathbf{\Phi}_{F_j}, \quad
    \mathbf{M}^R_j = \frac{\eta_j}{n_{R_j}} \mathbf{\Phi}_{R_j}^\top \mathbf{\Phi}_{R_j}.$

\begin{mythr}[NTK Extension]
\label{thm:NTK}
Fix $s \le t$. Suppose there exists a nonzero subspace $W \subset \mathbb{R}^d$ such that for all $j \in \{s, \dots, t\}$:
\vspace{-2mm}
\begin{enumerate}
    \item[\textbf{N1}] \emph{Invariance.} $\mathbf{A}_j W \subseteq W$.
    \item[\textbf{N2}] \emph{Forget-dominance on $W$.} $(\mathbf{M}^F_j - \lambda \mathbf{M}^R_j)|_W \succeq \rho \mathbf{I}_W$ for some $\rho > 0$.
\end{enumerate}
Then for all $\mathbf{w} \in W$,
\begin{equation}
    \|\mathbf{A}_{t:s}\, \mathbf{w}\| \;\ge\; (1+\rho)^{t-s+1}\|\mathbf{w}\|.
    \label{eq:ntk_growth}
\end{equation}
Moreover, the parameter deviation satisfies
\begin{equation}
\vspace{-2mm}
    \bm{\delta}_t
    \;=\;
    \mathbf{A}_{t:s}\,\bm{\delta}_{s-1}
    \;-\;
    \sum_{j=s}^{t} \mathbf{A}_{t:j+1}\bigl(\mathbf{b}^F_j - \lambda \mathbf{b}^R_j\bigr),
    \label{eq:ntk_closed_form}
\end{equation}
so boundedness of $\bm{\delta}_t$ along $W$ requires cancellation of the exponentially growing homogeneous component $\mathbf{A}_{t:s}\,\bm{\delta}_{s-1}$.
\end{mythr}

See Appendix~\ref{app:proof_3} for proof.

\subsection{Diagnostic Quantities}
\label{subsec:metrics}

Theorems~\ref{thm:operator_growth} and~\ref{thm:expanding_mode_general} establish that when unlearning tasks share a common expanding subspace $W$, the iterates must remain bounded through one of two compensating mechanisms, corresponding precisely to failure modes and in Definition~\ref{def:plasticity_collapse}. To detect these mechanisms empirically, we introduce two quantities that directly probe the geometry of parameter updates relative to $W$.

In practice, $W$ is estimated from the update increments $\{\mathbf{u}_1, \dots, \mathbf{u}_T\}$ by computing the top-$r$ right singular vectors of the stacked matrix $\mathbf{U} = [\mathbf{u}_1 \mid \cdots \mid \mathbf{u}_T] \in \mathbb{R}^{d \times T}$. Let $\mathbf{Q} \in \mathbb{R}^{d \times r}$ denote the resulting orthonormal basis, so that $\mathbf{P}_W = \mathbf{Q}\mathbf{Q}^\top$ is the orthogonal projector onto the estimated $W$.

\begin{mydef}[Energy Ratio and Projection Coefficient]
\label{def:ER_CO}
For each stage $t$, we define Energy Ratio (ER) and Projection Coefficient (CO) as follows:
\begin{align}
    \mathrm{ER}_t \;&:=\; \frac{\|\mathbf{P}_W \mathbf{u}_t\|^2}{\|\mathbf{u}_t\|^2} \;=\; \frac{\|\mathbf{Q} \mathbf{Q}^\top \mathbf{u}_t\|^2}{\|\mathbf{u}_t\|^2} \;, \label{eq:ER_def}\\[4pt]
    \mathrm{CO}_t \;&:=\; \mathbf{Q}^\top \mathbf{u}_t \;. \label{eq:CO_def}
\end{align}
\end{mydef}

$\mathrm{ER}_t$ measures what proportion of the update $\mathbf{u}_t$ lies within $W$: a value close to 1 indicates strong alignment with the shared subspace. $\mathrm{CO}_t$ records the signed projection of $\mathbf{u}_t$ onto each principal direction in $W$, capturing the direction and magnitude of engagement with $W$.

These two quantities connect directly to the two failure modes:
\begin{itemize}
    \item \textbf{Energy Ratio for diminishing forgetting quality.} If $\mathbf{A}_{t:s}$ grows exponentially along $W$ and the iterates remain stable, the algorithm must reduce the component of $\mathbf{u}_t$ along $W$, thereby causing $\mathrm{ER}_t$ toward $0$. Since future forget sets have features in $U_t \subseteq W$ by (\textbf{B1}), an update nearly orthogonal to $W$ induces diminishing forgetting quality for subsequent tasks.

    \item \textbf{Coefficient for spontaneous re-memorization.} Alternatively, the algorithm may inject updates with opposing signs along $W$. Because the same shared directions also carry contributions from earlier unlearning steps (captured in the forcing terms $c_j$), such cancellation can partially undo prior forgetting-induced parameter changes, thereby restoring performance on previously forgotten data.
\end{itemize}


\section{Experiments}
\label{sec:experiment}

\subsection{Experimental Setup}

We evaluate plasticity collapse on image classification tasks. Experiments are conducted on \textbf{Tiny-ImageNet}~\cite{le2015tiny} and \textbf{CIFAR-100}, using \textbf{VGG-16-BN} and \textbf{PreResNet-110} as backbone architectures. The results on CIFAR-100 can be found in Appendix~\ref{app:ex}. We study five representative unlearning methods: \textbf{Finetuning (FT)}, \textbf{NegGrad+}~\cite{neggrad}, \textbf{RandomLabeling (RL)}~\cite{Amnesiac2021}, \textbf{SalUn}~\cite{salun2024}, and \textbf{MUNBa}~\cite{munba}. 
We exclude FG-OrIU~\cite{feng2025fg} from our baselines due to its inherent limitation to ViT architectures. Specifically, FG-OrIU relies on parameter-efficient fine-tuning via LoRA, which restricts its applicability to ViT models and prevents straightforward extension to architectures such as VGG and ResNet.

We consider two forgetting scenarios. In \textbf{random data forgetting}, each task $F_t$ consists of 1\% of training data selected uniformly at random, producing forget spans $U_t$ with relatively unstructured overlap. In \textbf{class-wise forgetting}, each $F_t$ consists of all samples from a designated class, producing more structured and concentrated forget spans.

\begin{figure*}[!t]
\begin{center}

    \subfloat[FT]{\includegraphics[width=0.2\linewidth]{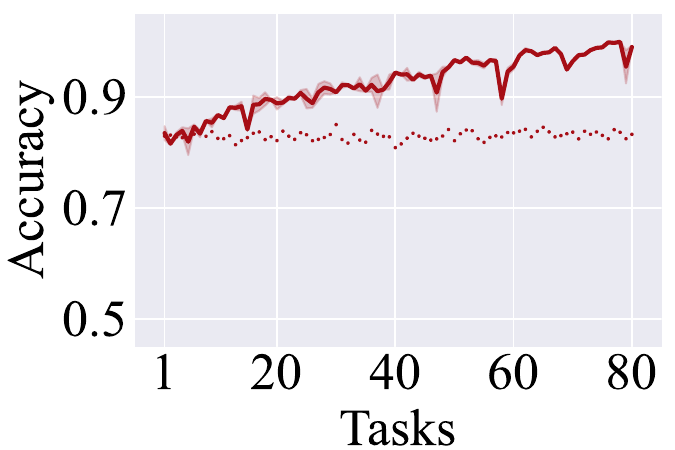}}
    \subfloat[NegGrad+]{\includegraphics[width=0.2\linewidth]{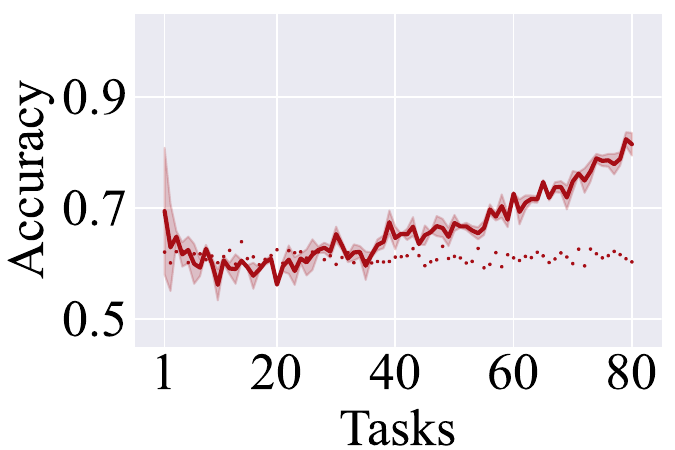}}
    \subfloat[RL]{\includegraphics[width=0.2\linewidth]{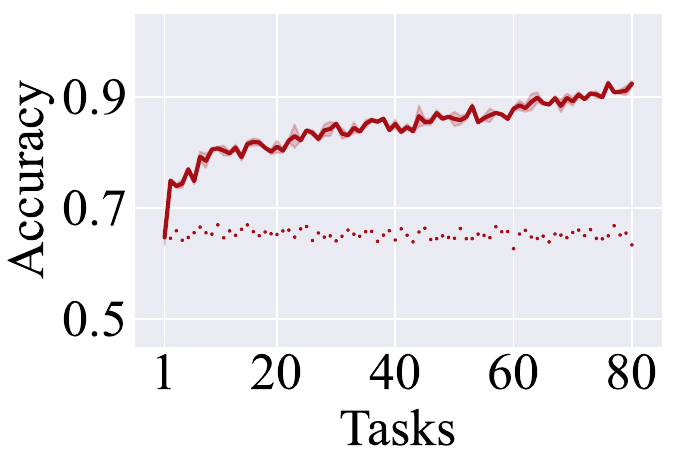}}
    \subfloat[SalUn]{\includegraphics[width=0.2\linewidth]{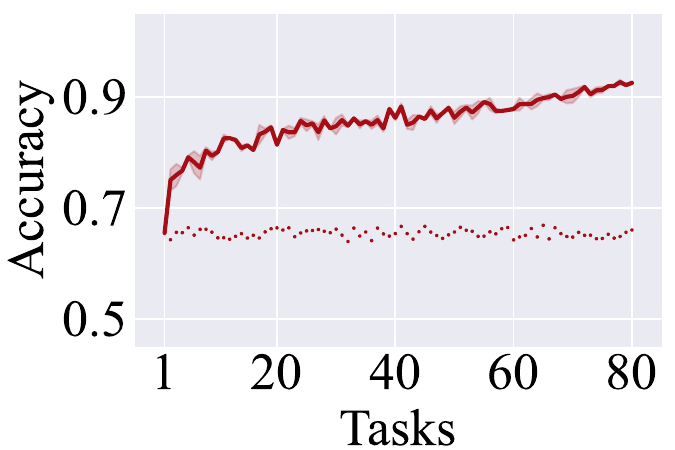}}
    \subfloat[MUNBa]{\includegraphics[width=0.2\linewidth]{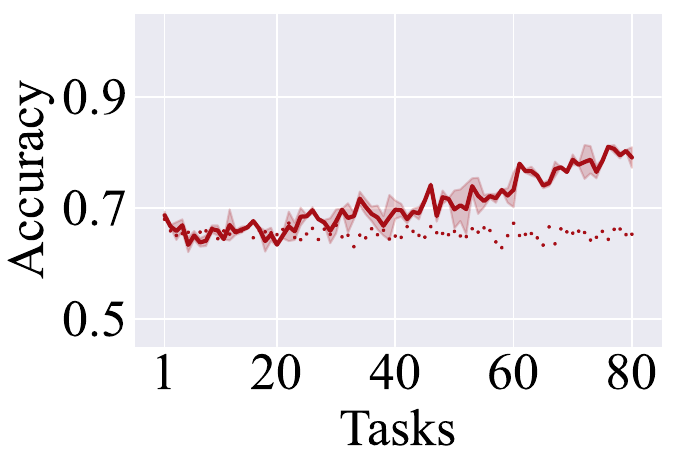}}

\vspace{-2mm}
\caption{Forgetting accuracy of various unlearning methods on Tiny-ImageNet with VGG-16-BN under the random data forgetting scenario. The solid line denotes the forgetting accuracy on each task, while the scatter points represent the forgetting accuracy for single-task unlearning. Lower accuracy indicates better forgetting quality. For all methods, the forgetting accuracy increases as the number of sequential unlearning tasks grows, suggesting a progressive degradation in forgetting quality.}
\label{fig:forget_acc_vgg_random}
\end{center}
\vspace{-4mm}
\end{figure*}

\begin{figure*}[!t]
\begin{center}

    \subfloat[FT]{\includegraphics[width=0.2\linewidth]{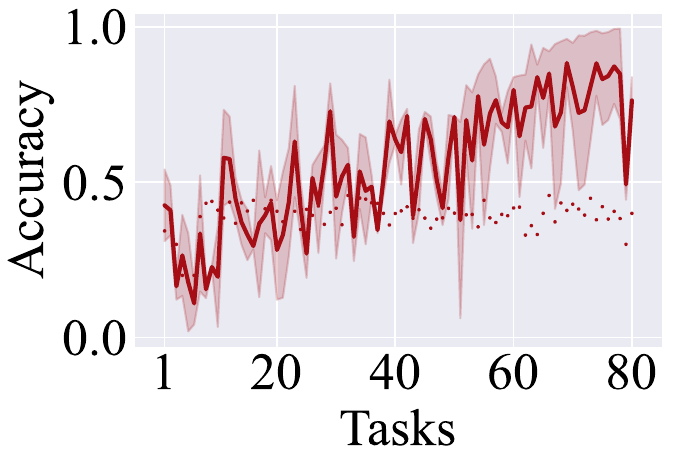}}
    \subfloat[NegGrad+]{\includegraphics[width=0.2\linewidth]{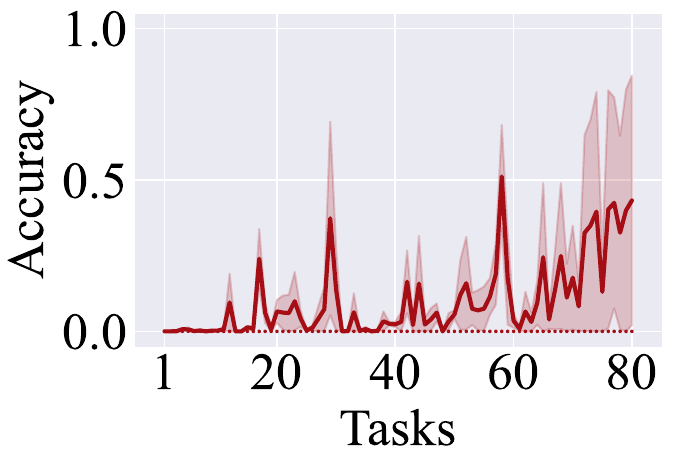}}
    \subfloat[RL]{\includegraphics[width=0.2\linewidth]{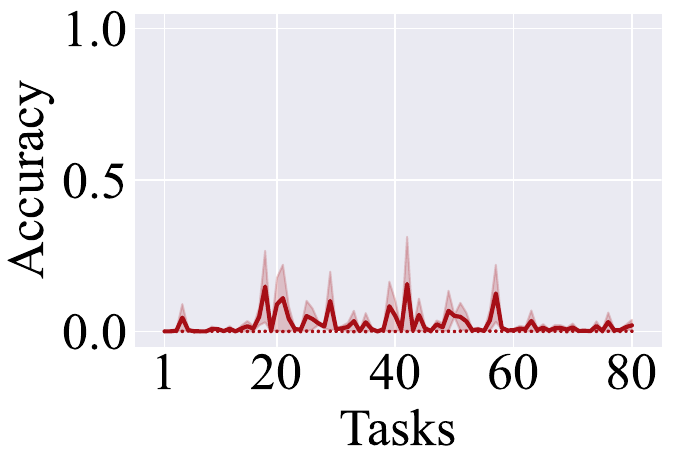}}
    \subfloat[SalUn]{\includegraphics[width=0.2\linewidth]{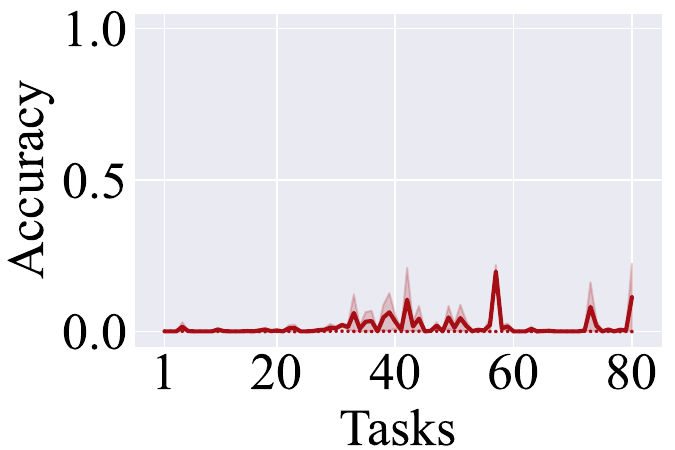}}
    \subfloat[MUNBa]{\includegraphics[width=0.2\linewidth]{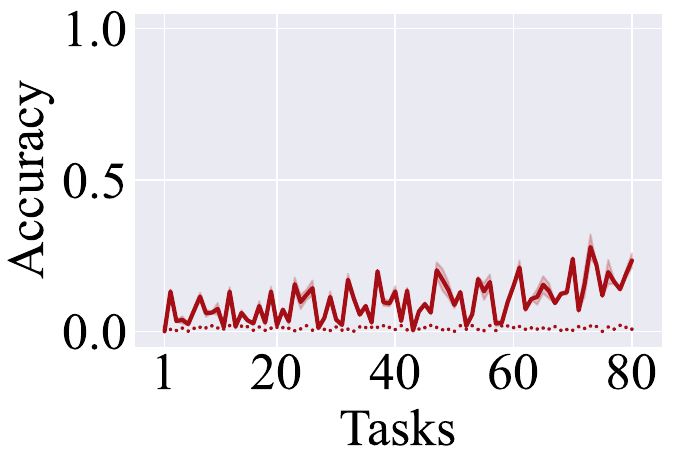}\label{fig:forget_acc_vgg_class_e}}

\vspace{-2mm}
\caption{Forgetting accuracy on Tiny-ImageNet with VGG-16-BN under the class-wise forgetting scenario. FT, NegGrad+ and MUNBa also show diminishing forgetting quality for subsequent tasks. RL and SalUn do not exhibit the clear forward failure in this setting, but we show its backward failure in Fig~\ref{fig:heatmap_class}.}
\label{fig:forget_acc_vgg_class}
\end{center}
\vspace{-8mm}
\end{figure*}

\subsection{Results}

\paragraph{Empirical Evidence for Two Failure Modes.}
The accuracy of the forget data is set as the primary metric for unlearning quality, as it is the most commonly used measure in prior work~\cite{huang2025unified,salun2024,zhao2024makes,munba}. Lower forgetting accuracy indicates higher forgetting quality. Specifically, at each stage $t$, we measure model accuracy on $F_t$ after the unlearning update. Results on retain and test accuracy are reported in Appendix~\ref{app:ex}. The results are average values from $3$ independent trials.

Figs.~\ref{fig:forget_acc_vgg_random} and~\ref{fig:forget_acc_vgg_class} plot forgetting accuracy across sequential tasks for all five methods under both scenarios. As a reference baseline, we also report the forgetting accuracy achieved by each method when applied to a single task in isolation (scatter points in Figs.~\ref{fig:forget_acc_vgg_random} and~\ref{fig:forget_acc_vgg_class}), allowing a direct comparison between single-task unlearning and continual unlearning to highlight the degradation induced by the sequential setting.

\begin{figure*}[!t]
\begin{center}

    \subfloat[FT]{\includegraphics[width=0.2\linewidth]{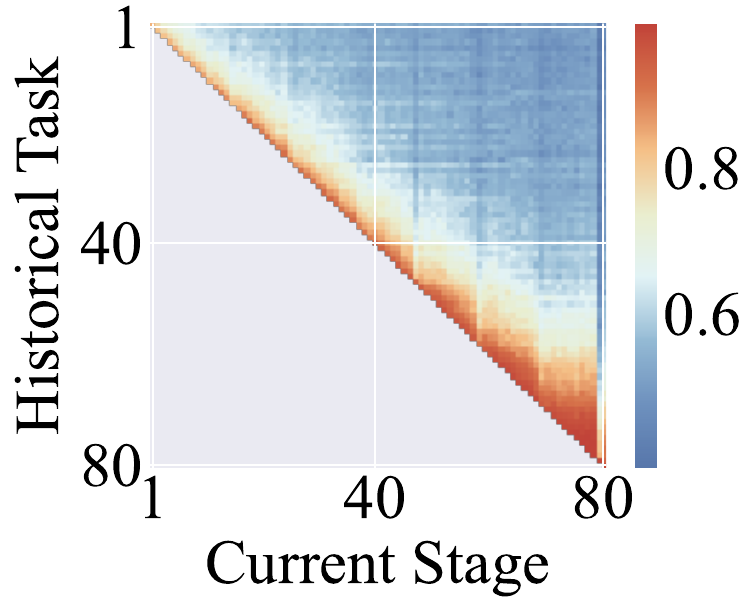}
    \label{fig:heatmap_random_a}}
    \subfloat[NegGrad+]{\includegraphics[width=0.2\linewidth]{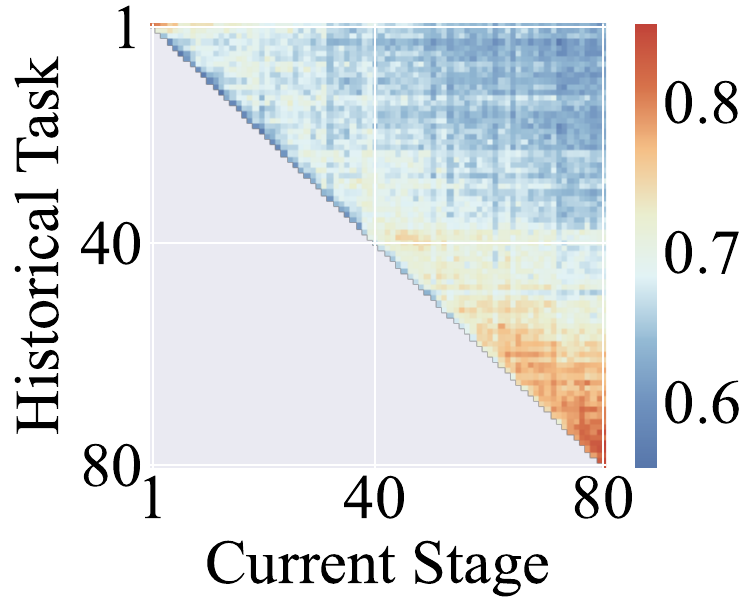}}
    \subfloat[RL]{\includegraphics[width=0.2\linewidth]{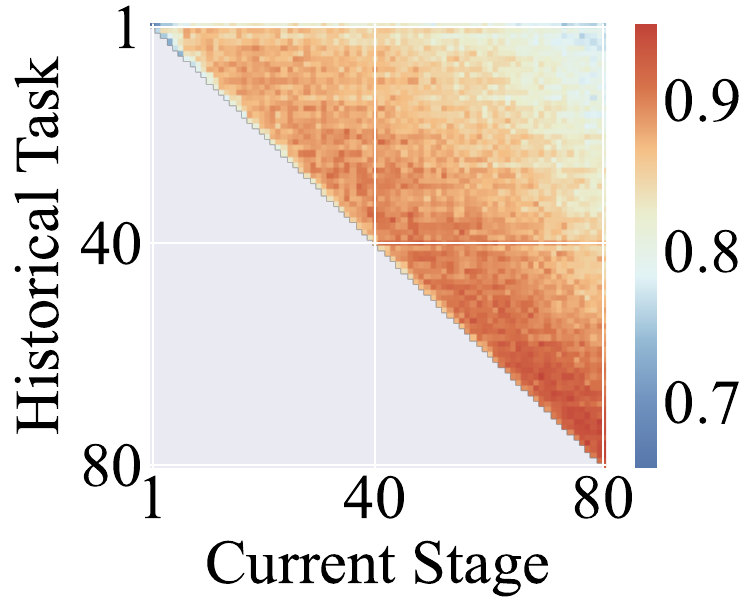}}
    \subfloat[SalUn]{\includegraphics[width=0.2\linewidth]{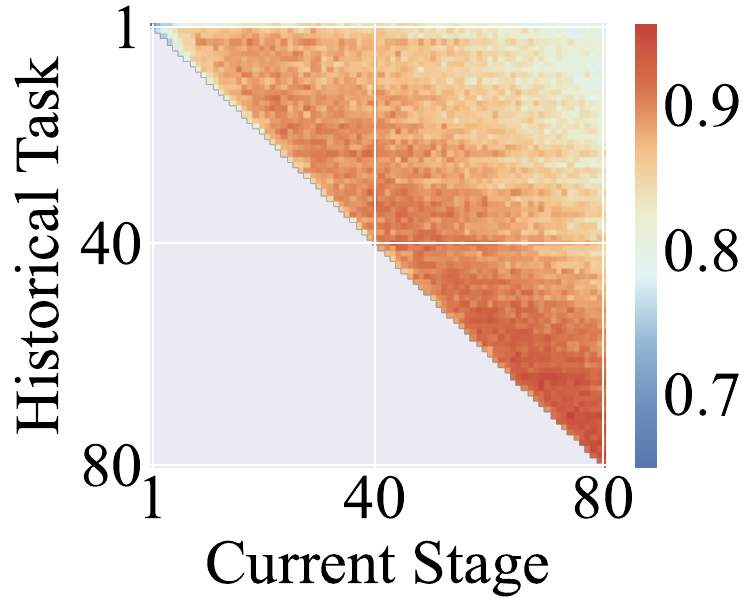}}
    \subfloat[MUNBa]{\includegraphics[width=0.2\linewidth]{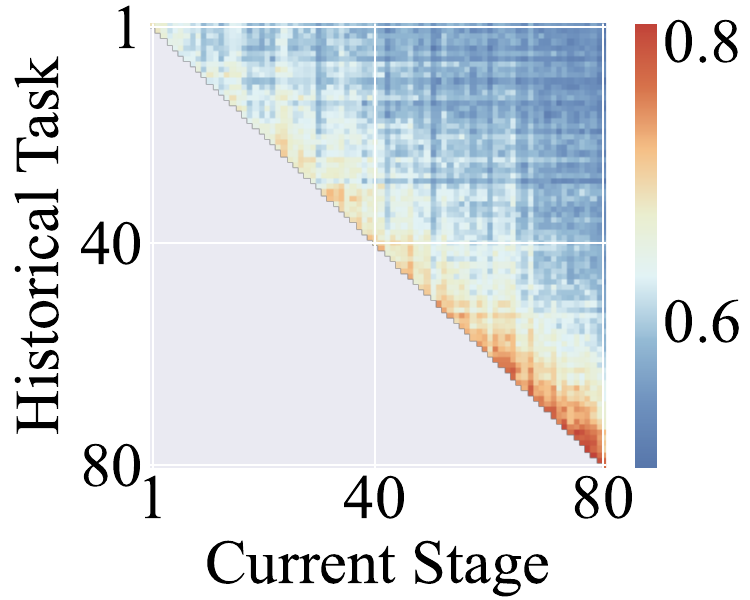}
    \label{fig:heatmap_random_e}}

\vspace{-2mm}
\caption{Forgetting accuracy heatmaps on Tiny-ImageNet with VGG-16-BN for historical tasks on the random data forgetting scenario. Entry $(i, j)$ records forgetting accuracy of task $F_i$ on the model $\bm{\theta}_j$ where $j \ge i$.}
\label{fig:heatmap_random}
\end{center}
\vspace{-6mm}
\end{figure*}
\begin{figure*}[!t]
\begin{center}

    \subfloat[FT]{\includegraphics[width=0.2\linewidth]{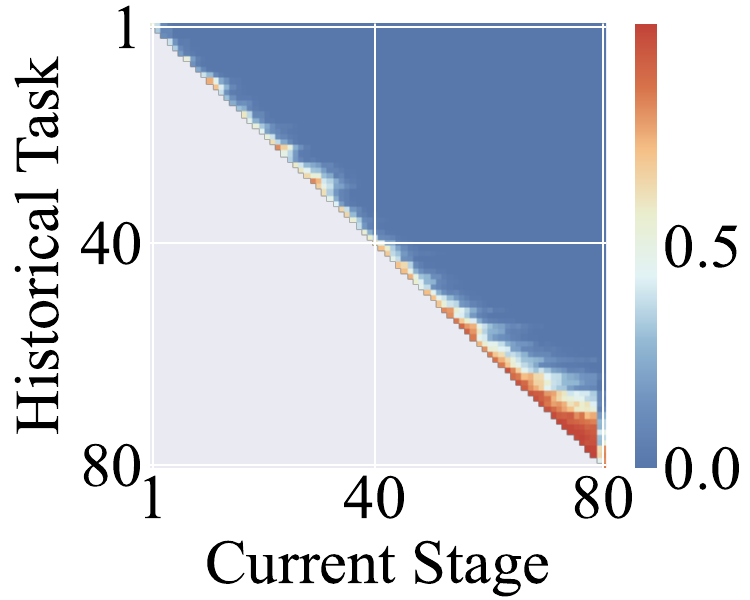}\label{fig:heatmap_class_a}}
    \subfloat[NegGrad+]{\includegraphics[width=0.2\linewidth]{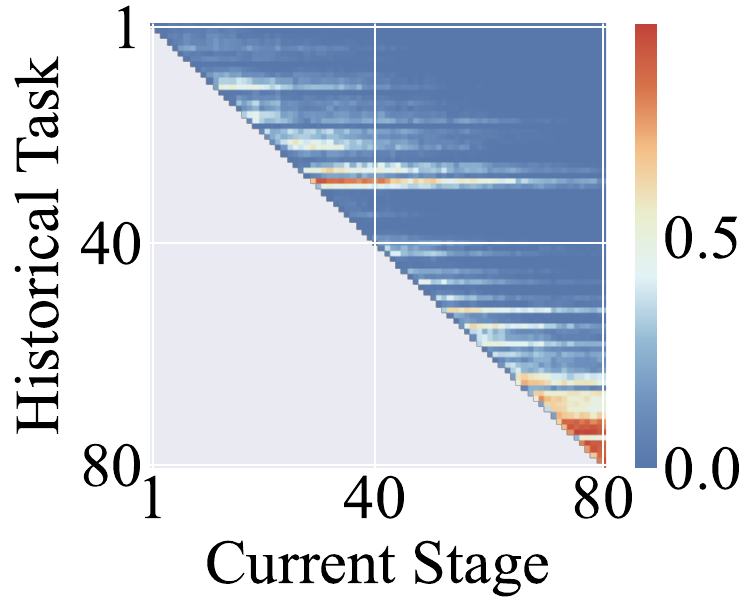}\label{fig:heatmap_class_b}}
    \subfloat[RL]{\includegraphics[width=0.2\linewidth]{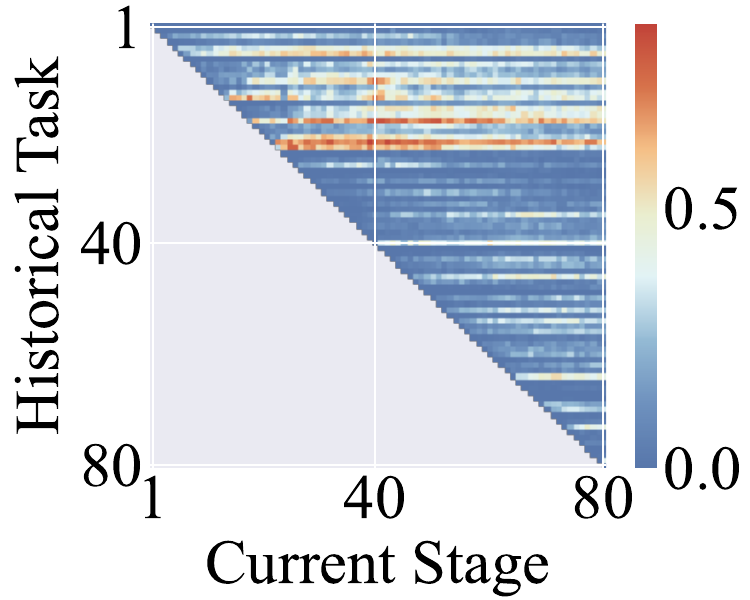}\label{fig:heatmap_class_c}}
    \subfloat[SalUn]{\includegraphics[width=0.2\linewidth]{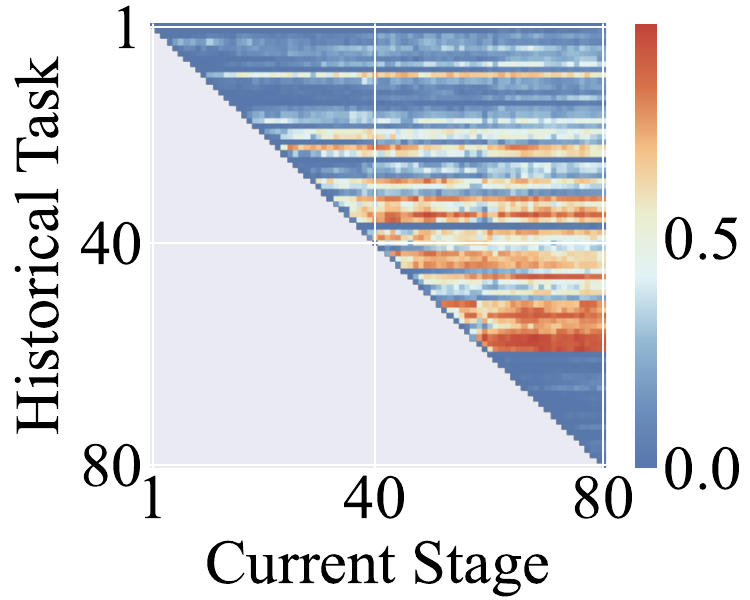}\label{fig:heatmap_class_d}}
    \subfloat[MUNBa]{\includegraphics[width=0.2\linewidth]{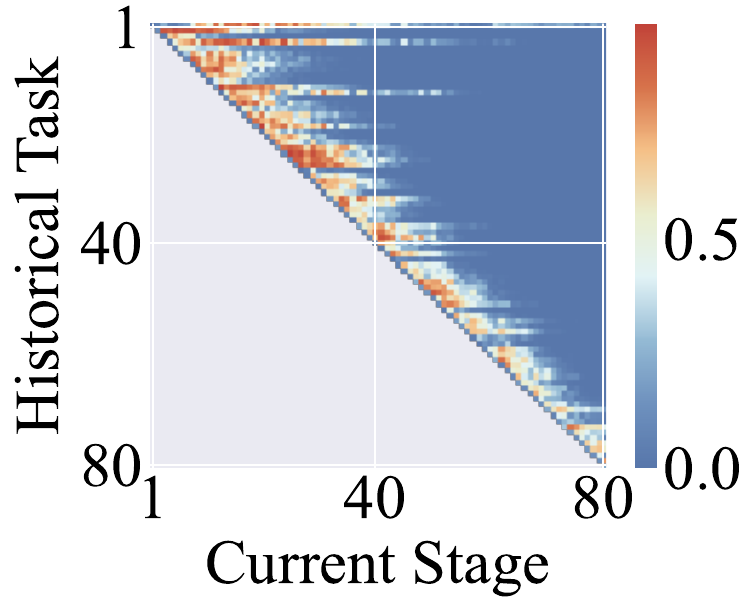}\label{fig:heatmap_class_e}}

\vspace{-2mm}
\caption{\footnotesize{Forgetting accuracy heatmaps on Tiny-ImageNet with VGG-16-BN for historical tasks on the class-wise forgetting scenario. RL and SalUn, which do not exhibit forward failure, show obvious backward failure in this setting.}}
\label{fig:heatmap_class}
\end{center}
\vspace{-8mm}
\end{figure*}
Under \textbf{random data forgetting} as shown in Fig.~\ref{fig:forget_acc_vgg_random}, all five methods exhibit increasing forgetting accuracy as tasks accumulate. This trend empirically validates our forward failure mode of plasticity collapse. On average across methods, forgetting accuracy of the last task in the continual setting is approximately 33.2\% higher than in the single-task baseline, with NegGrad+ showing the most severe degradation. Notably, even MUNBa, one of the most recent unlearning methods, is also not immune to this effect.

Under \textbf{class-wise forgetting} as shown in Fig.~\ref{fig:forget_acc_vgg_class}, the situation is more nuanced. Since class-wise forgetting is an inherently easier task, most existing methods achieve near-0\% forgetting accuracy in the single-task setting. In the continual setting, however, this strong performance no longer holds. FT, NegGrad+, and MUNBa all show clearly increasing forgetting accuracy as tasks accumulate, confirming forward failure mode. RL and SalUn exhibit more moderate increases and do not display a pronounced forward failure under this scenario. However, as we show next, both methods suffer from a backward failure mode, i.e., spontaneous re-memorization.

Figs.~\ref{fig:heatmap_random} and~\ref{fig:heatmap_class} track historical task forgetting accuracy over time, where entry $(i, j)$ records forgetting accuracy of task $F_i$ on the model $\bm{\theta}_j$ where $j \ge i$. The heatmap results are shown for a single random 
seed to facilitate the CO analysis presented later. The diagonal entries, representing current-task forgetting quality, show a consistent blue-to-red gradient as $t$ increases across Figs.~\ref{fig:heatmap_random_a}--\ref{fig:heatmap_random_e} and Figs.~\ref{fig:heatmap_class_a},~\ref{fig:heatmap_class_b},~\ref{fig:heatmap_class_e}, confirming progressive forgetting quality degradation, i.e., \textbf{forward failure}. The off-diagonal entries above the diagonal reveal the \textbf{backward failure} mode. Tasks that were initially well-forgotten (blue on diagonal entries) later transition to red for RL (Fig.~\ref{fig:heatmap_class_c}) and SalUn (Fig.~\ref{fig:heatmap_class_d}) under class-wise forgetting. Taking SalUn as a concrete example (Fig.~\ref{fig:heatmap_class_d}), most of the tasks in the index range 35-60 exhibit pronounced re-memorization, even though these tasks were well forgotten at their current stages. 

Importantly, the two failure modes are not mutually exclusive. MUNBa, for instance, displays both forward (Fig.~\ref{fig:forget_acc_vgg_class_e}) and backward (Fig.~\ref{fig:heatmap_class_e}) failure under class-wise forgetting scenario, demonstrating that a single method can suffer from both forms of plasticity collapse. \textbf{Taken together, these results show that plasticity collapse (both forward and backward failure modes) is a universal phenomenon, manifesting across different unlearning methods and forgetting scenarios.}

\begin{figure*}[t]
\begin{center}

    \subfloat[FT]{\includegraphics[width=0.2\linewidth]{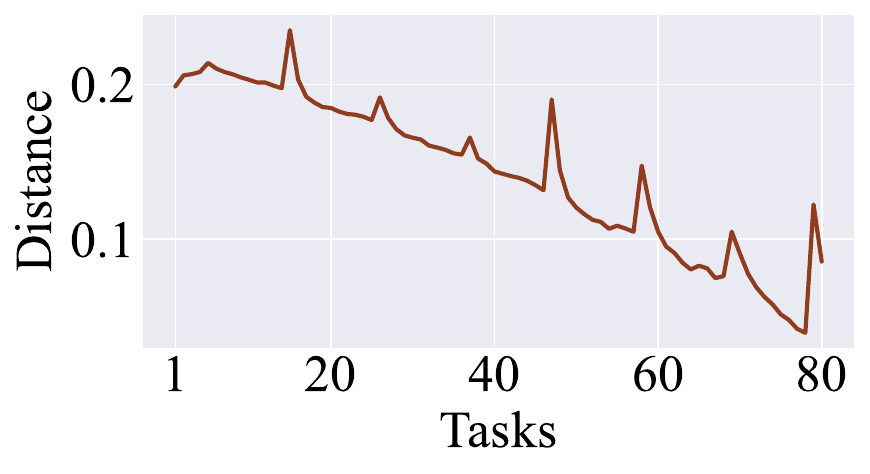}}
    \subfloat[NegGrad+]{\includegraphics[width=0.2\linewidth]{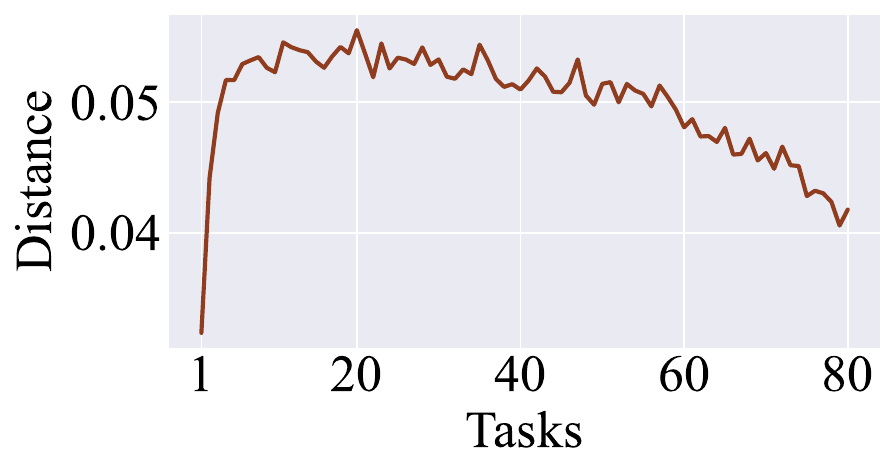}}
    \subfloat[RL]{\includegraphics[width=0.2\linewidth]{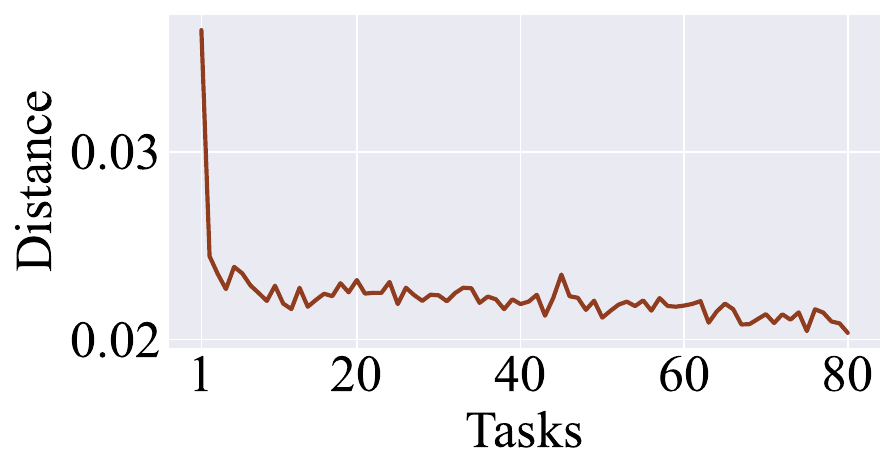}}
    \subfloat[SalUn]{\includegraphics[width=0.2\linewidth]{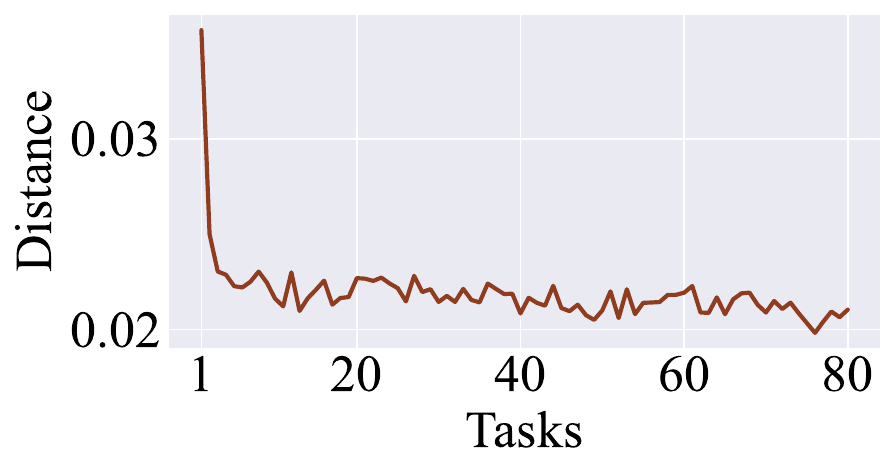}}
    \subfloat[MUNBa]{\includegraphics[width=0.2\linewidth]{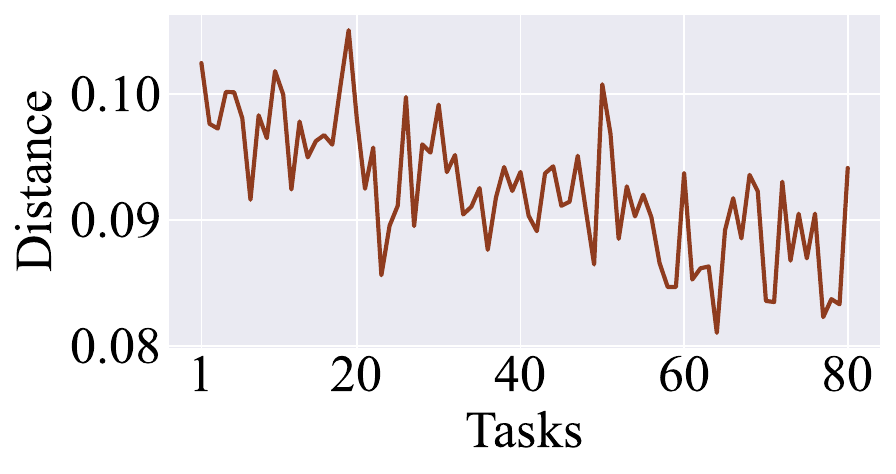}}

\vspace{-2mm}
\caption{L2 norm of parameter differences between consecutive updates on Tiny-ImageNet with VGG-16-BN under the random data forgetting scenario. It shows an obvious decrease in the magnitude of parameter updates as the number of tasks increases, indicating that the model becomes increasingly frozen in the parameter space.}
\label{fig:l2_vgg_random}
\end{center}
\vspace{-8mm}
\end{figure*}

\begin{figure*}[t]
\begin{center}

    \subfloat[FT]{\includegraphics[width=0.2\linewidth]{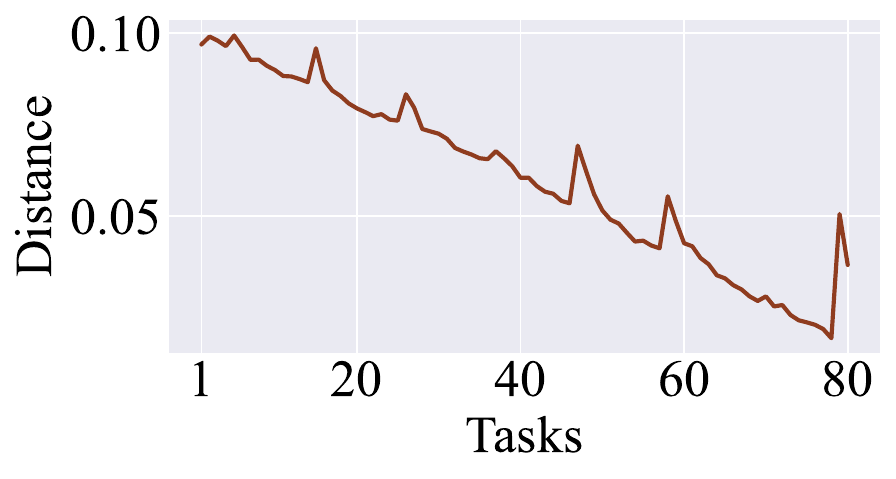}}
    \subfloat[NegGrad+]{\includegraphics[width=0.2\linewidth]{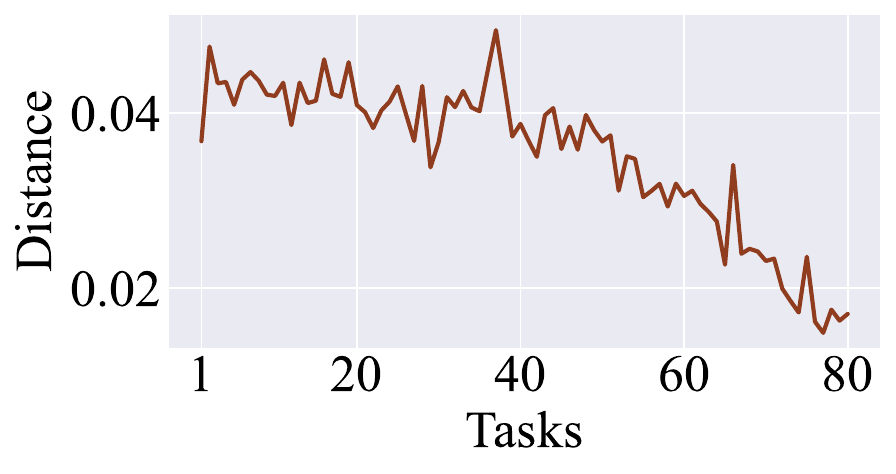}}
    \subfloat[RL]{\includegraphics[width=0.2\linewidth]{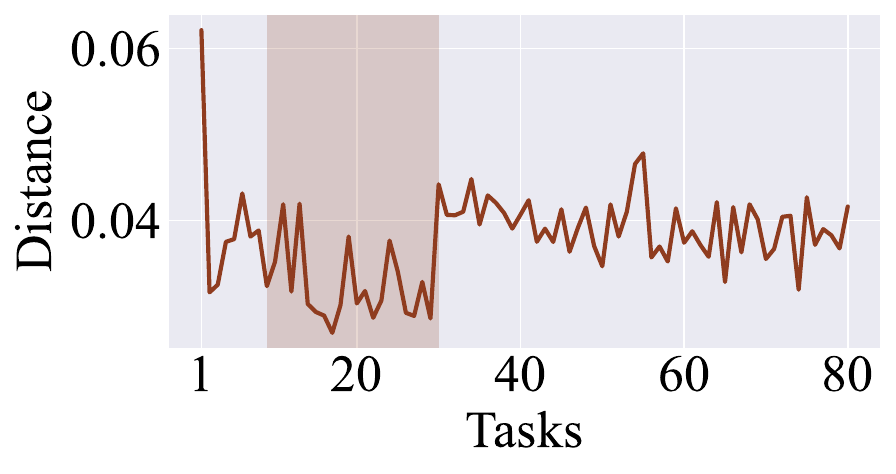}\label{fig:l2_vgg_class_c}}
    \subfloat[SalUn]{\includegraphics[width=0.2\linewidth]{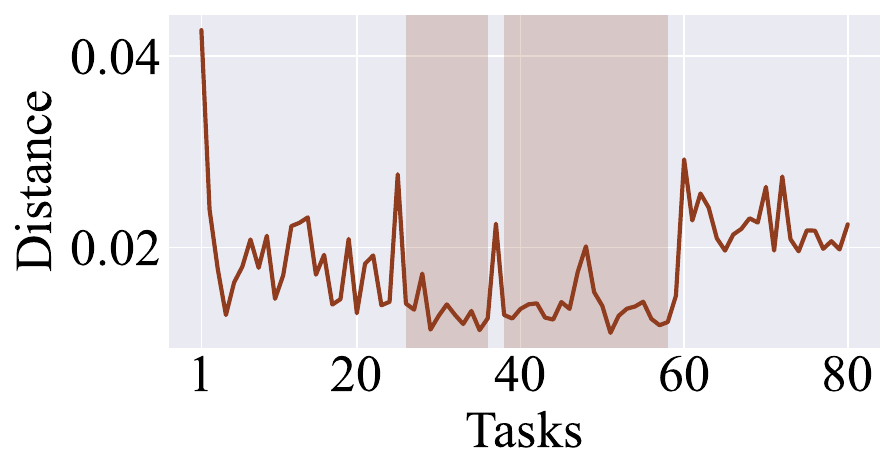}\label{fig:l2_vgg_class_d}}
    \subfloat[MUNBa]{\includegraphics[width=0.2\linewidth]{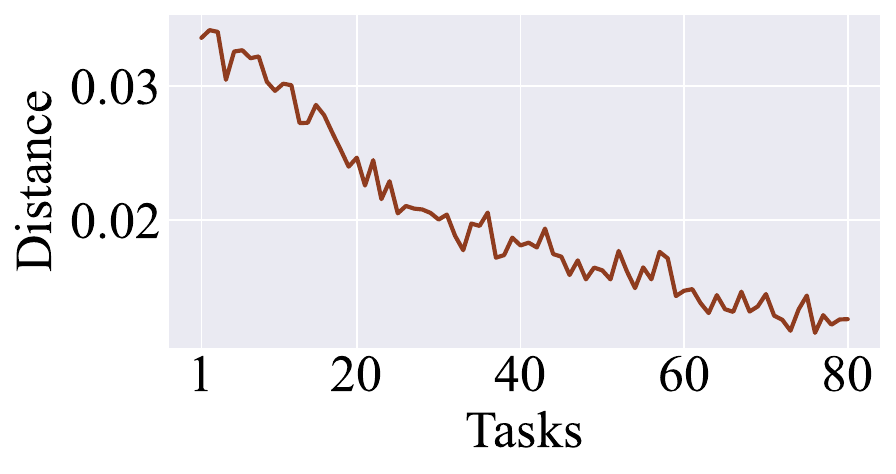}\label{fig:l2_vgg_class_e}}

\vspace{-2mm}
\caption{L2 norm of parameter differences between consecutive updates on Tiny-ImageNet with VGG-16-BN under the class-wise forgetting scenario. For RL and SalUn, which exhibit backward failure, the tasks where re-memorization occurs show lower L2 norm values compared to other tasks.}
\label{fig:l2_vgg_class}
\end{center}
\vspace{-8mm}
\end{figure*}

\paragraph{Update magnitude diminishes over time.}
Theorems~\ref{thm:operator_growth} and~\ref{thm:expanding_mode_general} predict that as $W$ saturates, the suppression mechanism increasingly drives parameter updates away from $W$. This phenomenon is expected to manifest as a systematic reduction in the update magnitude, defined as $\mathbf{u}_t := \bm{\theta}_{t} - \bm{\theta}_{t-1}$. Figs.~\ref{fig:l2_vgg_random} and~\ref{fig:l2_vgg_class} empirically corroborate this. $\|\mathbf{u}_t\|$ declines consistently in scenarios exhibiting forward failure, with the most pronounced decay aligning precisely with the stages where forgetting quality degrades most severely. Furthermore, for cases characterized by backward failure (Figs.~\ref{fig:l2_vgg_class_c}-\ref{fig:l2_vgg_class_d}), tasks exhibiting the re-memorization phenomenon (shaded region) also display significantly lower $L_2$ norms compared to others. \textbf{These observations provide direct evidence that the model becomes increasingly frozen in the parameter space, progressively losing its capacity to generate the meaningful updates required to fulfill new unlearning requests.}

\begin{figure*}[t]
\begin{center}
    \subfloat[FT]{\includegraphics[width=0.2\linewidth]{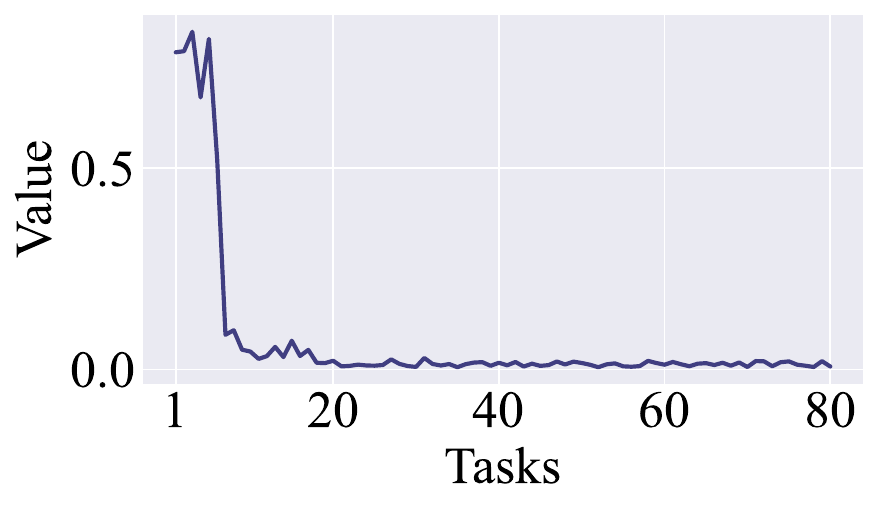}}
    \subfloat[NegGrad+]{\includegraphics[width=0.2\linewidth]{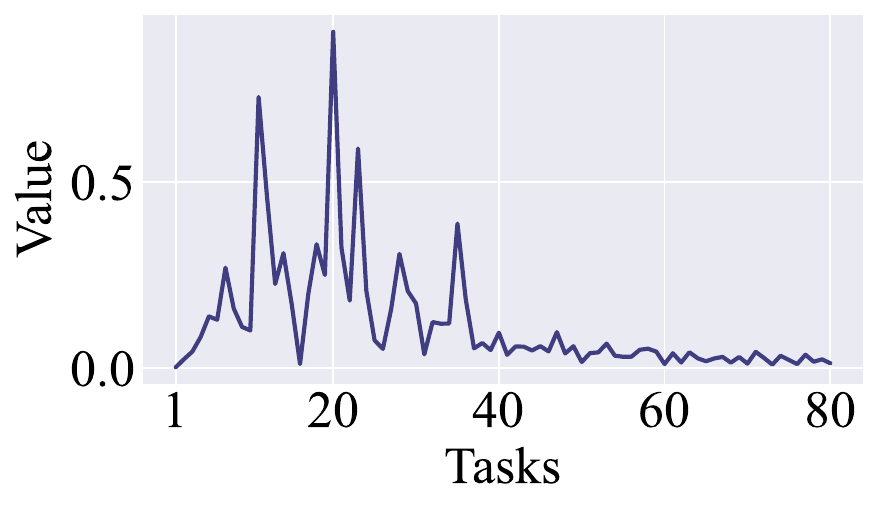}}
    \subfloat[RL]{\includegraphics[width=0.2\linewidth]{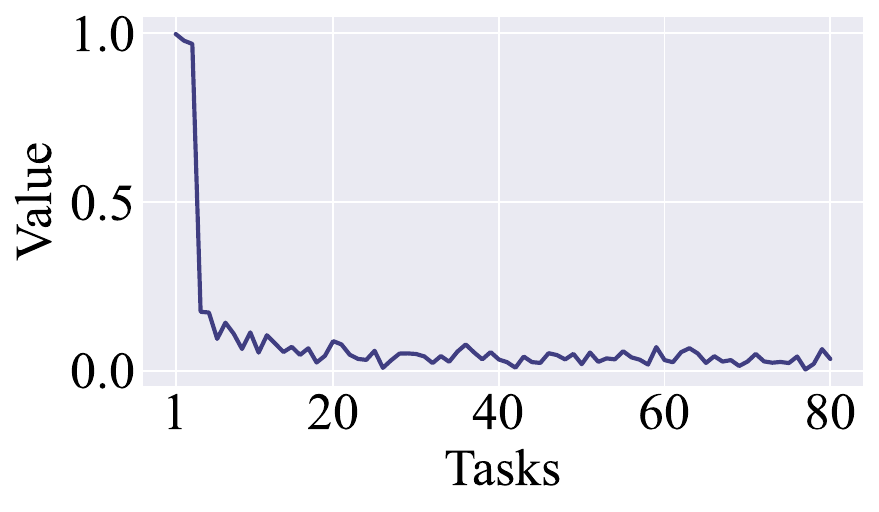}}
    \subfloat[SalUn]{\includegraphics[width=0.2\linewidth]{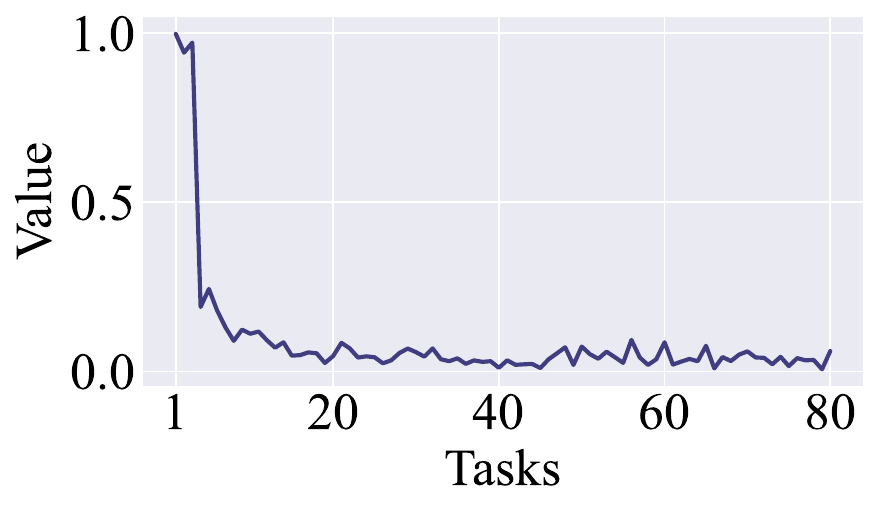}}
    \subfloat[MUNBa]{\includegraphics[width=0.2\linewidth]{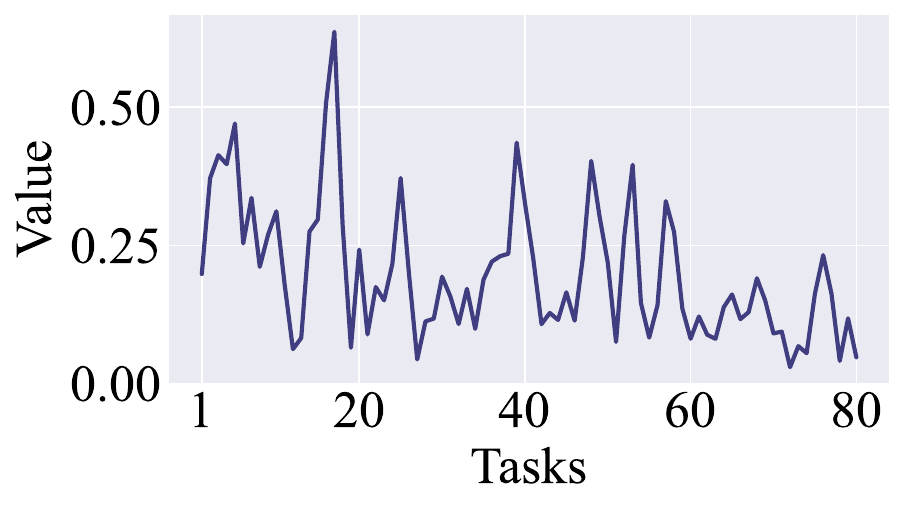}}

\vspace{-2mm}
\caption{Energy Ratio for forward failure diagnosis on Tiny-ImageNet with VGG-16-BN under the random data forgetting scenario.}
\label{fig:vgg_random_energy}
\end{center}
\vspace{-6mm}
\end{figure*}
\begin{figure*}[t]
\begin{center}
    \subfloat[FT]{\includegraphics[width=0.2\linewidth]{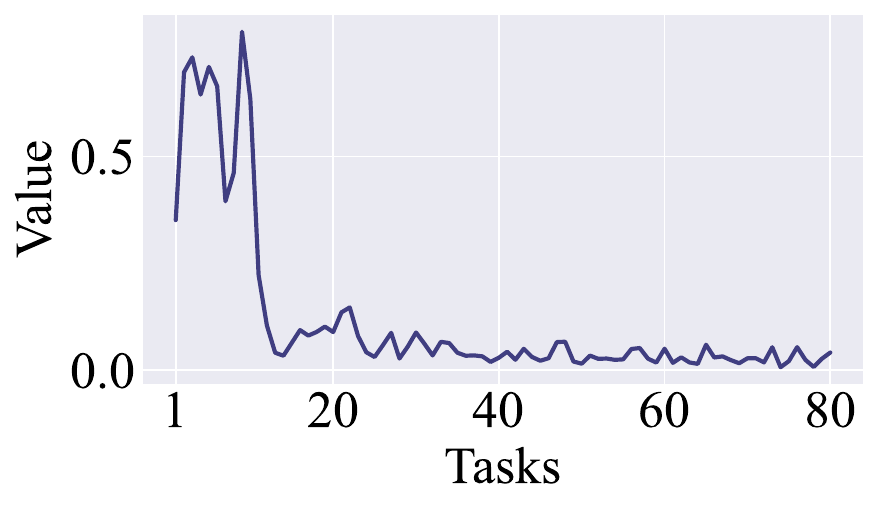}}
    \subfloat[NegGrad+]{\includegraphics[width=0.2\linewidth]{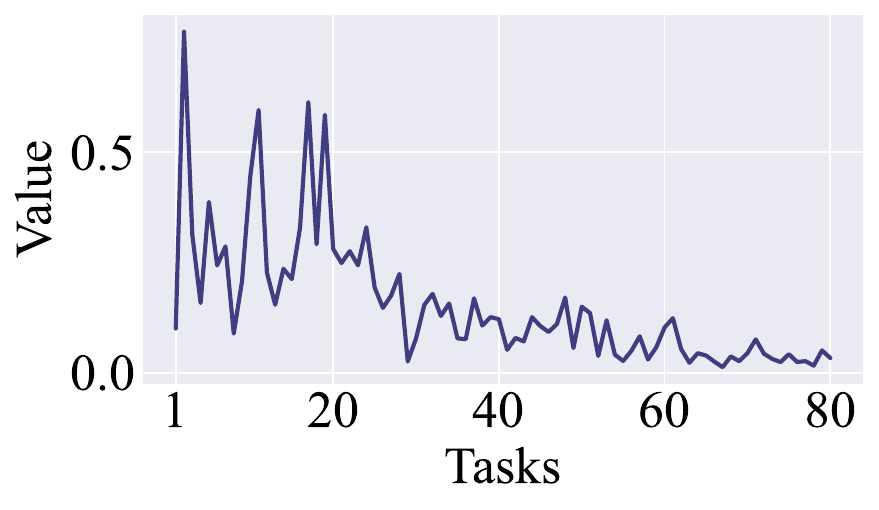}}
    \subfloat[MUNBa]{\includegraphics[width=0.2\linewidth]{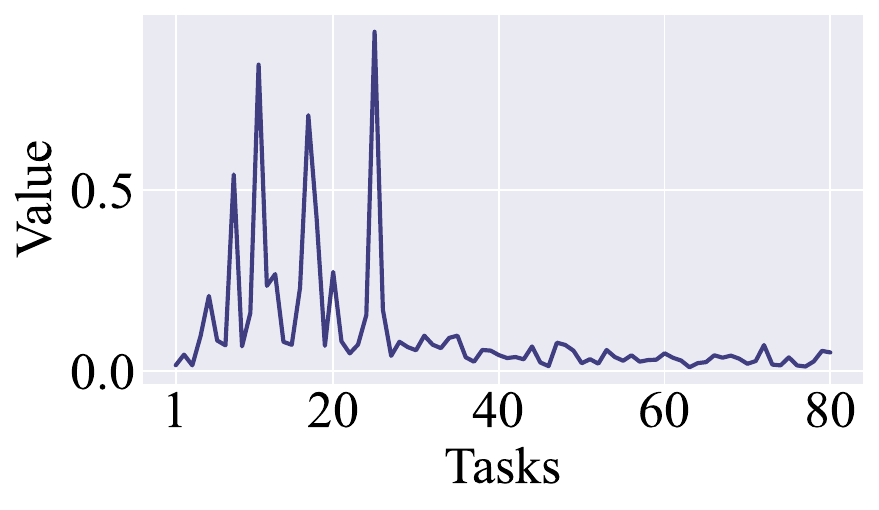}}

\vspace{-2mm}
\caption{Energy Ratio for forward failure diagnosis on Tiny-ImageNet with VGG-16-BN under the class-wise forgetting scenario.}
\label{fig:vgg_class_energy}
\end{center}
\vspace{-8mm}
\end{figure*}

\paragraph{Diagnostic Quantities for Plasticity Collapse.}
We now apply $\mathrm{ER}_t$ and $\mathrm{CO}_t$ in Section~\ref{subsec:metrics} to directly validate the theoretical predictions. The subspace $W$ is estimated from $\{\mathbf{u}_1, \dots, \mathbf{u}_T\}$ via PCA of the stacked update matrix $\mathbf{U}$, with $\mathbf{Q}$ taken as the leading $r = 1$ right singular vector for visualization clarity. Patterns are consistent for $r > 1$.

Figs.~\ref{fig:vgg_random_energy} and~\ref{fig:vgg_class_energy} plot $\mathrm{ER}_t$ across sequential tasks. In the cases with \textbf{forward failure}, $\mathrm{ER}_t$ decays progressively toward zero, confirming the suppression mechanism: as $\mathbf{A}_{t:s}$ accumulates exponential growth along $W$, it is forced to route updates away from $W$ to maintain stability. The decay of $\mathrm{ER}_t$ closely tracks the degradation of forgetting accuracy, providing quantitative evidence that geometric saturation of $W$ is the proximate cause of forward failure.

\begin{wrapfigure}{r}{0.45\linewidth}
\vspace{-5mm}

\centering
\subfloat[RL]{%
    \includegraphics[width=0.48\linewidth]{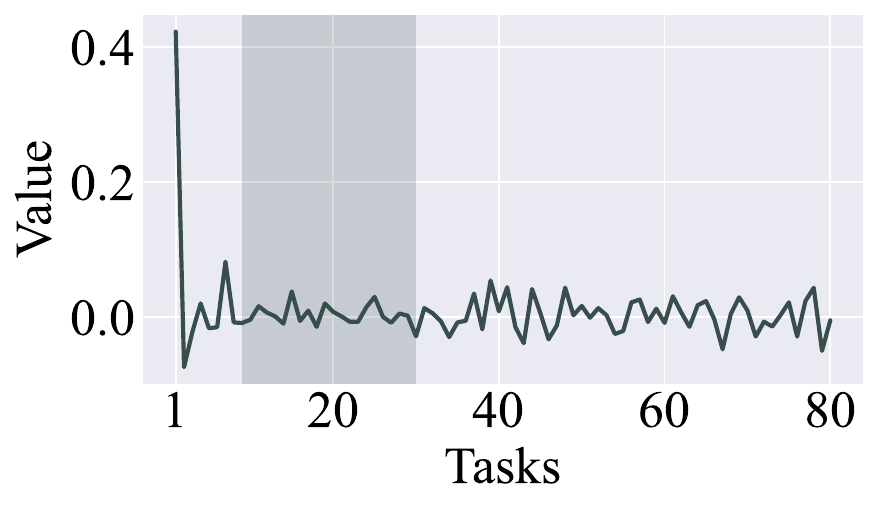}
}
\subfloat[SalUn]{%
    \includegraphics[width=0.48\linewidth]{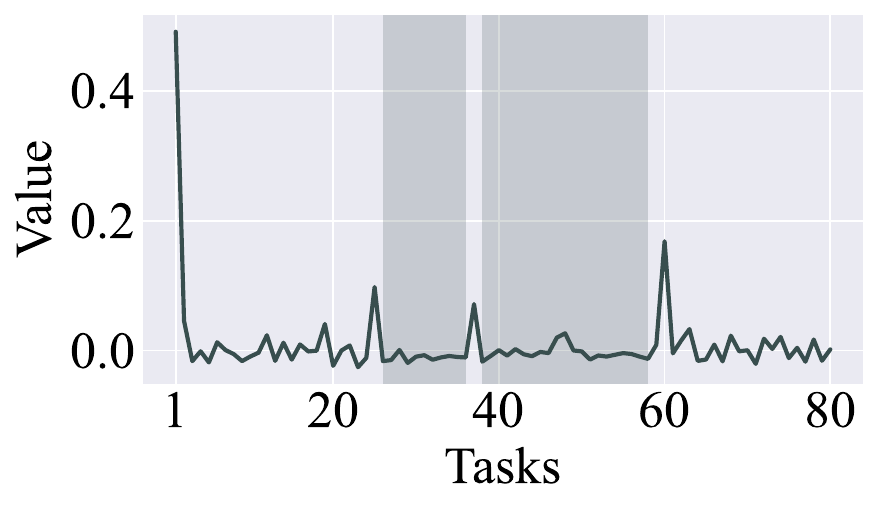}
}

\vspace{-2mm}
\caption{Coefficient for backward failure diagnosis on Tiny-ImageNet with VGG-16-BN under the class-wise forgetting scenario.}
\label{fig:vgg_class_coeffs}

\vspace{-6mm}
\end{wrapfigure}
Fig.~\ref{fig:vgg_class_coeffs} plots $\mathrm{CO}_t$ for methods exhibiting re-memorization. Tasks undergoing re-memorization show $\mathrm{CO}_t \approx 0$. Their updates fail to engage the principal forgetting direction, indicating the model cannot make progress along $W$. Tasks surrounding re-memorization events display pronounced sign oscillations in $\mathrm{CO}_t$, alternating between large positive and negative values. These oscillations directly reflect antagonistic interference: forgetting a new task requires injecting updates along $W$ that partially undo prior forgetting operations, triggering re-memorization of earlier tasks.

Together, the decay of $\mathrm{ER}_t$ and the oscillatory behavior of $\mathrm{CO}_t$ provide consistent, theory-grounded evidence for both failure modes of plasticity collapse, closing the loop between the theoretical predictions of Theorems~\ref{thm:operator_growth}--\ref{thm:expanding_mode_general} and the empirical observations.

\section{Conclusion}
\label{sec:conclusion}
We identify loss of unlearning plasticity as a fundamental challenge in continual machine unlearning. Through theoretical analysis and empirical validation, we demonstrate that models undergoing sequential unlearning operations exhibit deteriorating forgetting quality and spontaneous re-memorization. These findings reveal that current unlearning methods, designed for single-shot scenarios, face severe limitations in realistic continual settings, necessitating fundamentally new plasticity-preserving approaches for practical deployment.

Our theoretical analysis provides actionable guidance for plasticity-preserving unlearning methods in the future. Specifically, the Energy Ratio and Projection Coefficient introduced in Section~\ref{subsec:metrics} can serve as lightweight online monitors for detecting impending collapse, and can be naturally incorporated as regularization signals into existing gradient-based unlearning objectives. We hope this work serves as both a cautionary analysis and a constructive foundation for the next generation of continual unlearning algorithms.



\section*{Acknowledgments}
This work was supported in part by the National Science Foundation under grants IIS-2246157, FMitF-2319243, and the Department of Energy under grant DE-CR0000042. The project was also supported by computational resources provided by the NSF ACCESS and Argonne National Lab.

%
%
\bibliographystyle{splncs04}
\bibliography{main}

\newpage
\appendix
\onecolumn


\newpage
\appendix
\onecolumn
\section*{Appendix}

\section{Unlearning Methods}
The details of the baselines in our work are as follows,
\begin{itemize}
    \item \textbf{Finetuning (FT)}~\cite{finetune2021} fine-tunes the pre-trained model $\theta_o$ on the retain dataset.

    \item \textbf{NegGrad+}~\cite{neggrad} addresses Gradient Ascent’s issue by combining fine-tuning on $\train_r$ and gradient ascent on $\train_f$.

    \item \textbf{RandomLabeling (RL)}~\cite{Amnesiac2021} fine-tunes the model on the forgetting dataset $\train_f$ using randomly assigned labels to enforce forgetting.

    \item \textbf{SalUn}~\cite{salun2024} performs unlearning by optimizing only the salient parameters of the model identified from the randomly labeled forgetting data.

    \item \textbf{MUNBa}~\cite{munba} poses machine unlearning as a cooperative bargaining game between forgetting and preservation players, using a Nash bargaining closed-form solution to resolve gradient conflicts and achieve an optimal Pareto stationary point.

\end{itemize}

\section{Proof for Theorem~\ref{thm:operator_growth}}
\label{app:proof_1}
\begin{proof}
\textbf{Step 1 (one-step bound).}
Fix $j \in \mathcal{T}$ and $\mathbf{w} \in W$. Since $\mathbf{M}_j \succeq \mathbf{0}$, we have $\|\mathbf{M}_j \mathbf{w}\|^2 \ge 0$. By (\textbf{B2}),
\[
    \|(\mathbf{I}+\mathbf{M}_j)\mathbf{w}\|^2
    = \|\mathbf{w}\|^2 + 2\mathbf{w}^\top \mathbf{M}_j \mathbf{w} + \|\mathbf{M}_j \mathbf{w}\|^2
    \ge \|\mathbf{w}\|^2 + 2\rho\|\mathbf{w}\|^2
    \ge (1+\rho)^2\|\mathbf{w}\|^2,
\]
so $\|(\mathbf{I}+\mathbf{M}_j)\mathbf{w}\| \ge (1+\rho)\|\mathbf{w}\|$.

\textbf{Step 2 (iteration).}
Define $\mathbf{w}_s := \mathbf{v}$ and $\mathbf{w}_{j+1} := (\mathbf{I}+\mathbf{M}_j)\mathbf{w}_j$ for $j = s, \dots, t$.
By (\textbf{B1}), $\mathbf{M}_j \mathbf{w}_j \in W$ whenever $\mathbf{w}_j \in W$, so $\mathbf{w}_{j+1} \in W$ by closure. Applying the one-step bound at each $j \in \mathcal{T}$ yields $\|\mathbf{w}_{t+1}\| \ge (1+\rho)^{|\mathcal{T}|}\|\mathbf{v}\|$. Since $\mathbf{w}_{t+1} = \mathbf{A}_{t:s}\,\mathbf{v}$, this gives Eq.~\ref{eq:op_growth}.
\end{proof}

\section{Proof for Theorem~\ref{thm:NTK}}
\label{app:proof_3}

\begin{proof}
Under the NTK linearization, $\mathbf{M}^F_j$ and $\mathbf{M}^R_j$ are symmetric positive semidefinite by construction. By (\textbf{N2}),
\[
    \mathbf{A}_j|_W = \mathbf{I}_W + (\mathbf{M}^F_j - \lambda \mathbf{M}^R_j)|_W \succeq (1+\rho)\mathbf{I}_W.
\]
Hence the smallest singular value of $\mathbf{A}_j|_W$ is at least $1+\rho$, which implies
\[
    \|\mathbf{A}_j \mathbf{w}\| \ge (1+\rho)\|\mathbf{w}\| \qquad \forall\, \mathbf{w} \in W.
\]
By invariance (\textbf{N1}), all intermediate iterates remain in $W$, so repeated application of the one-step bound yields Eq.~\ref{eq:ntk_growth}. The closed-form expression Eq.~\ref{eq:ntk_closed_form} follows by unrolling the affine recursion, identical to the proof of Lemma~\ref{lem:closed_form}.
\end{proof}

\section{Additional Experiments}
\label{app:ex}

\subsection{Accuracy Results}
\begin{figure*}[!t]
\begin{center}

    \subfloat[FT]{\includegraphics[width=0.2\linewidth]{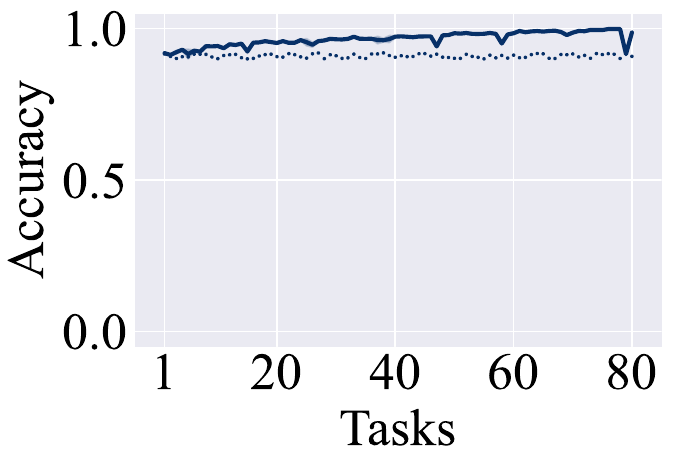}}
    \subfloat[NegGrad+]{\includegraphics[width=0.2\linewidth]{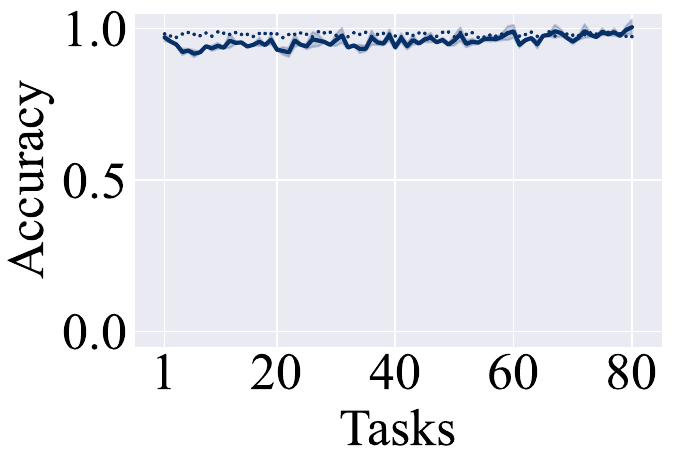}}
    \subfloat[RL]{\includegraphics[width=0.2\linewidth]{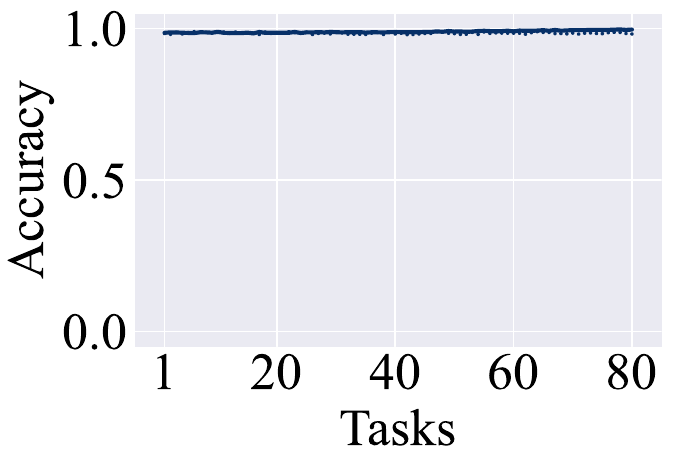}}
    \subfloat[SalUn]{\includegraphics[width=0.2\linewidth]{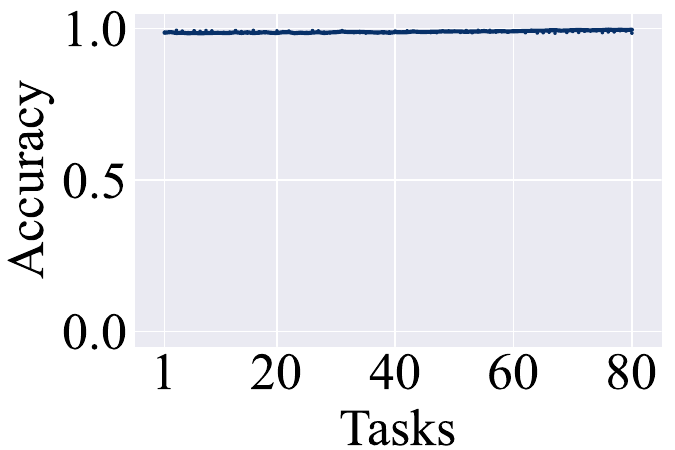}}
    \subfloat[MUNBa]{\includegraphics[width=0.2\linewidth]{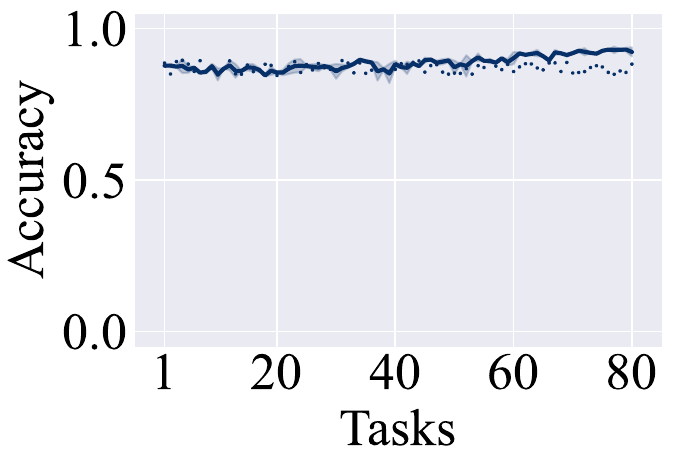}}

\vspace{-2mm}
\caption{Retain accuracy of various unlearning methods on Tiny-ImageNet with VGG-16-BN under the random data forgetting scenario.}
\label{fig:retain_acc_vgg_random}
\end{center}
\vspace{-6mm}
\end{figure*}

\begin{figure*}[!t]
\begin{center}

    \subfloat[FT]{\includegraphics[width=0.2\linewidth]{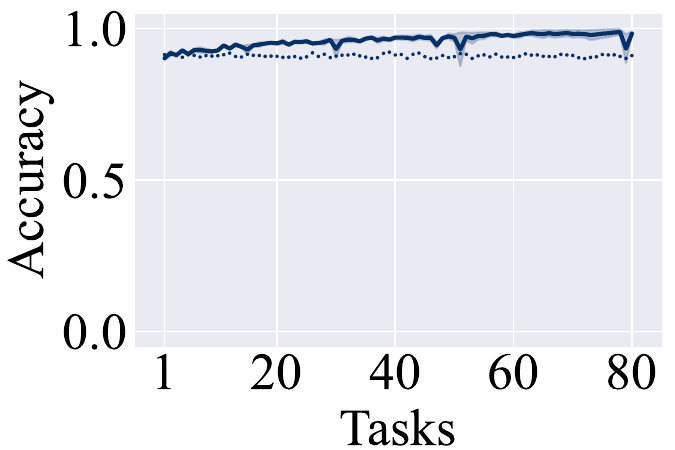}}
    \subfloat[NegGrad+]{\includegraphics[width=0.2\linewidth]{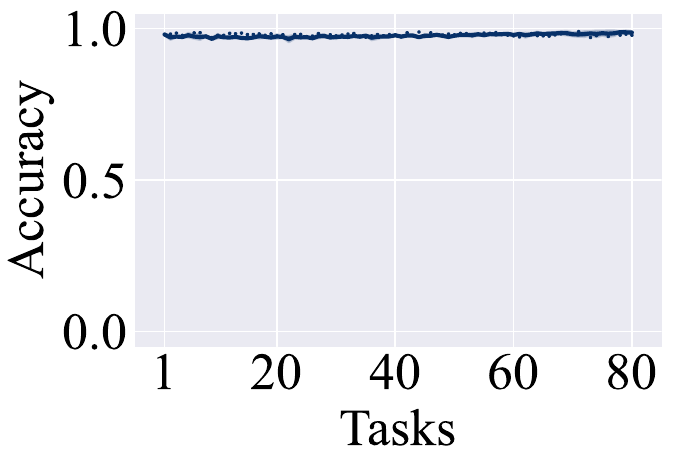}}
    \subfloat[RL]{\includegraphics[width=0.2\linewidth]{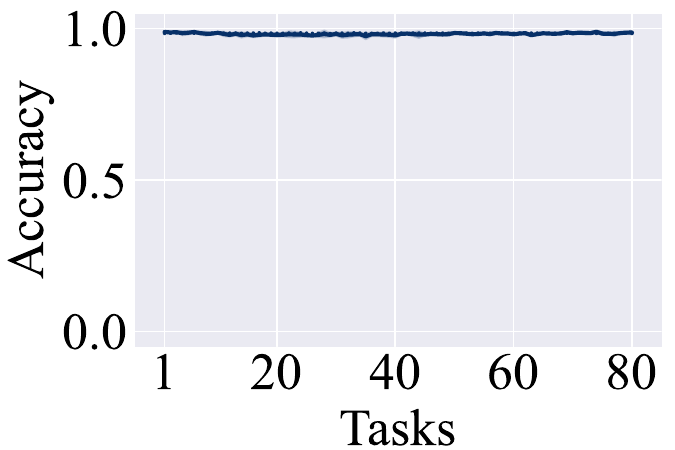}}
    \subfloat[SalUn]{\includegraphics[width=0.2\linewidth]{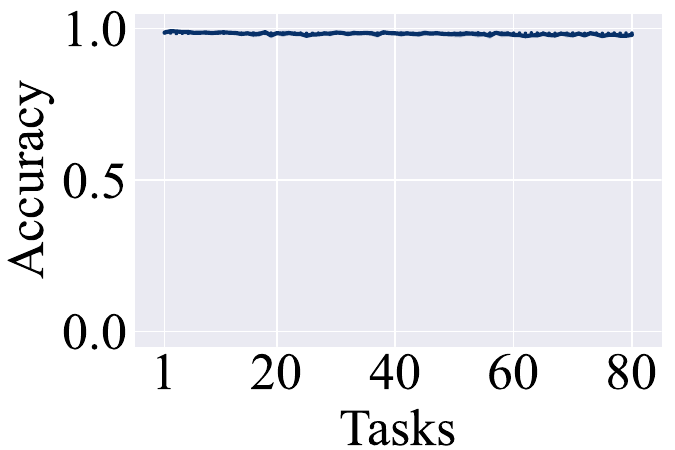}}
    \subfloat[MUNBa]{\includegraphics[width=0.2\linewidth]{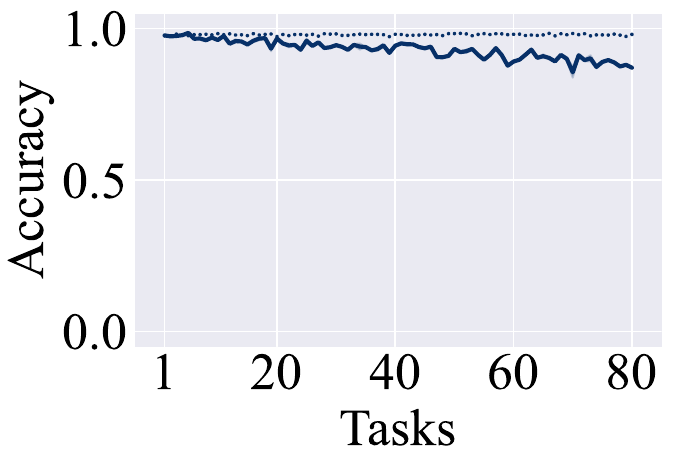}}

\vspace{-2mm}
\caption{Retain accuracy on Tiny-ImageNet with VGG-16-BN under the class-wise forgetting scenario.}
\label{fig:retain_acc_vgg_class}
\end{center}
\vspace{-6mm}
\end{figure*}
\begin{figure*}[!t]
\begin{center}

    \subfloat[FT]{\includegraphics[width=0.2\linewidth]{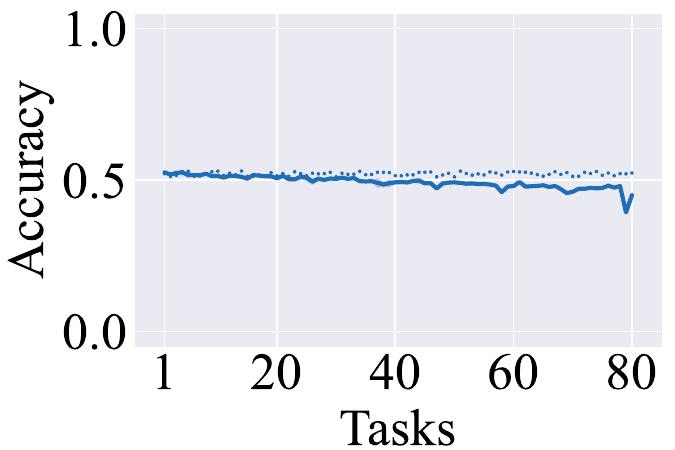}}
    \subfloat[NegGrad+]{\includegraphics[width=0.2\linewidth]{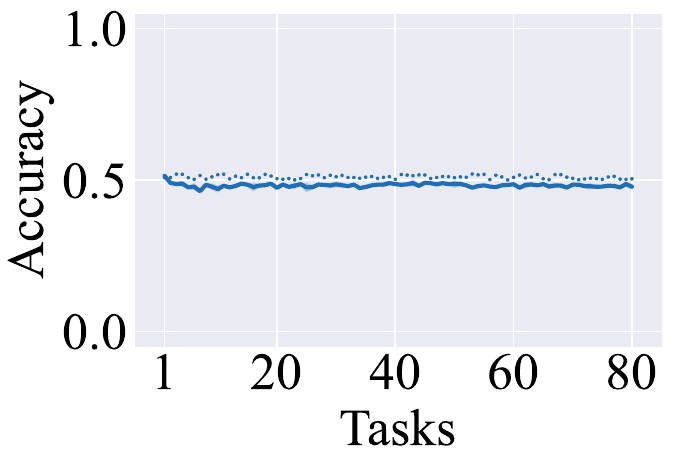}}
    \subfloat[RL]{\includegraphics[width=0.2\linewidth]{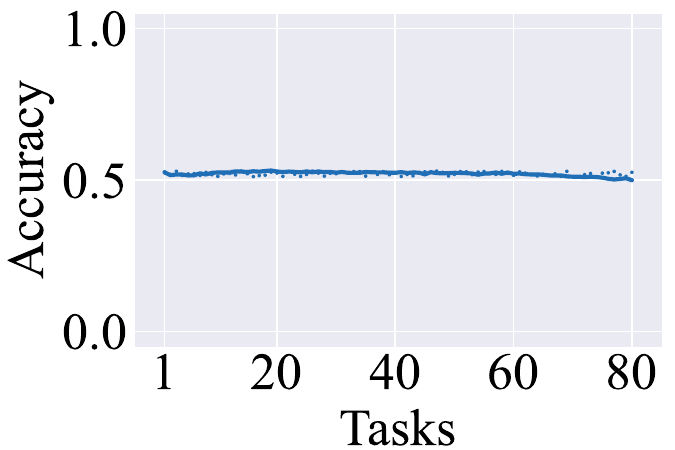}}
    \subfloat[SalUn]{\includegraphics[width=0.2\linewidth]{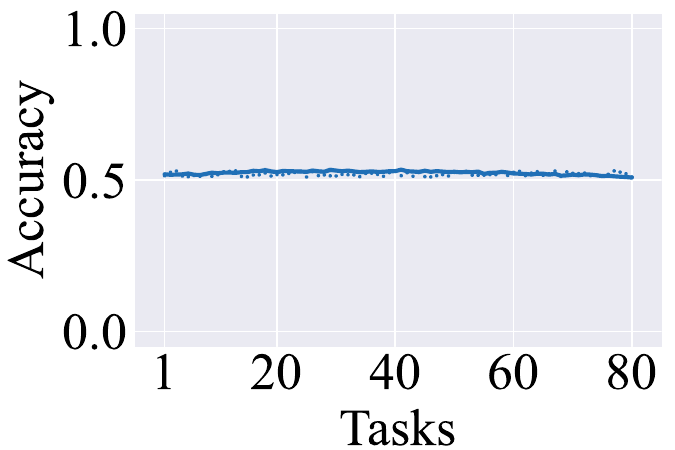}}
    \subfloat[MUNBa]{\includegraphics[width=0.2\linewidth]{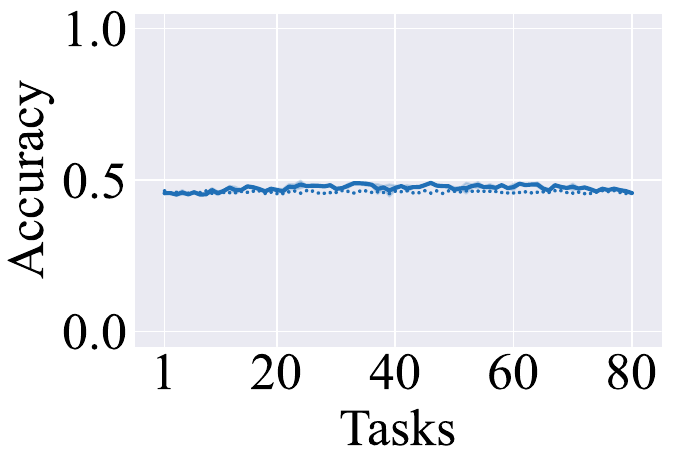}}

\vspace{-2mm}
\caption{Test accuracy of various unlearning methods on Tiny-ImageNet with VGG-16-BN under the random data forgetting scenario.}
\label{fig:test_acc_vgg_random}
\end{center}
\vspace{-6mm}
\end{figure*}

\begin{figure*}[!t]
\begin{center}

    \subfloat[FT]{\includegraphics[width=0.2\linewidth]{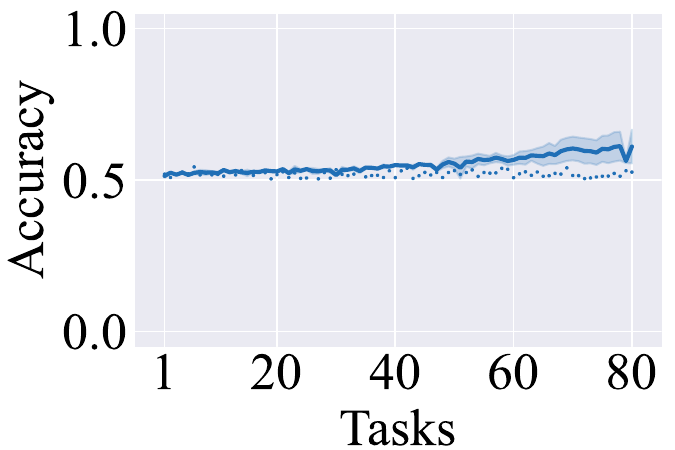}}
    \subfloat[NegGrad+]{\includegraphics[width=0.2\linewidth]{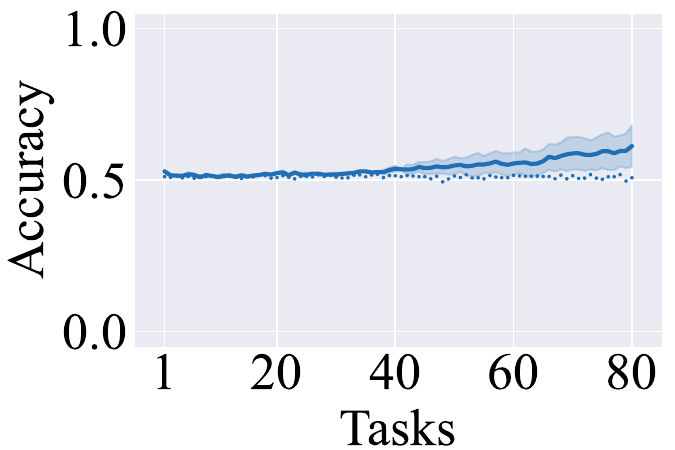}}
    \subfloat[RL]{\includegraphics[width=0.2\linewidth]{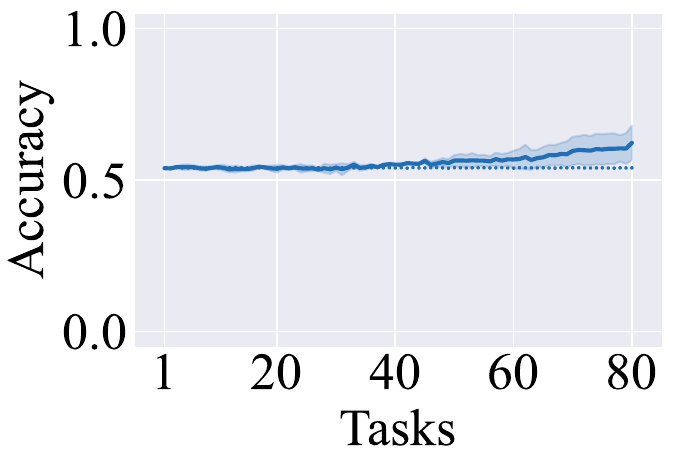}}
    \subfloat[SalUn]{\includegraphics[width=0.2\linewidth]{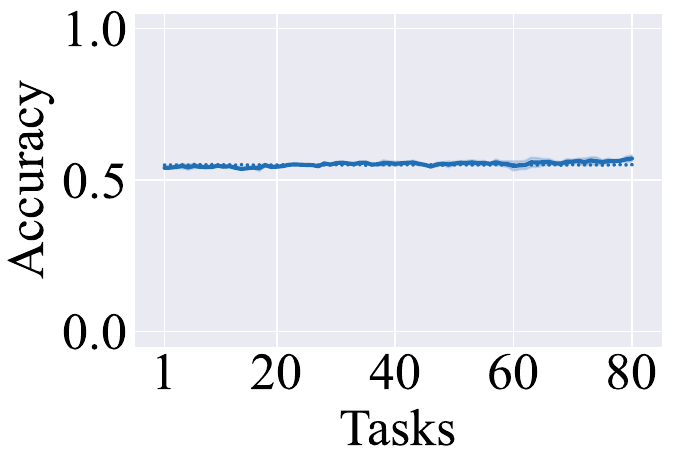}}
    \subfloat[MUNBa]{\includegraphics[width=0.2\linewidth]{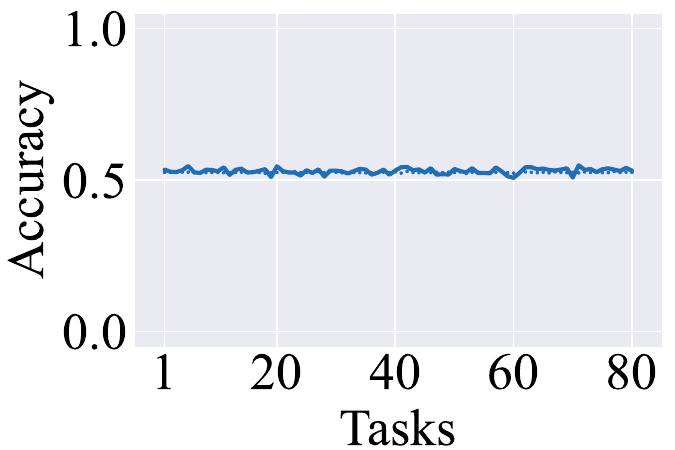}}

\vspace{-2mm}
\caption{Test accuracy on Tiny-ImageNet with VGG-16-BN under class-wise forgetting scenario.}
\label{fig:test_acc_vgg_class}
\end{center}
\vspace{-6mm}
\end{figure*}

We report the retain accuracy and test accuracy in Figs~\ref{fig:retain_acc_vgg_random}-\ref{fig:test_acc_vgg_class}. Across most evaluated cases, the models successfully maintain stable retain and test performance, thereby preserving overall model utility. This consistency aligns with the fundamental assumption of our theoretical framework.

Results for the CIFAR-100 dataset are presented in Figs~\ref{fig:forget_acc_resnet_random}-\ref{fig:test_acc_resnet_class}. In the random data forgetting scenario, FT and NegGrad+ exhibit significant degradation in model utility. Since our theorem assumes that utility remains relatively stable, these two methods fall outside the scope of our plasticity collapse analysis and are subsequently excluded from further theoretical validation. Among the remaining methods, SalUn and RL clearly demonstrate forward failure. Conversely, MUNBa fails to exhibit plasticity collapse in this setting because its single-shot baseline already shows negligible forgetting effectiveness.

In the class-wise forgetting scenario, FT and MUNBa similarly lack the prerequisite single-shot unlearning efficacy, thus failing to trigger observable plasticity collapse. NegGrad+ exhibits forward failure under this setting, while RL and SalUn demonstrate pronounced backward failure, as illustrated by the heatmaps in Figs~\ref{fig:heatmap_class_c}-\ref{fig:heatmap_class_d}.

\begin{figure*}[!t]
\begin{center}

    \subfloat[FT]{\includegraphics[width=0.2\linewidth]{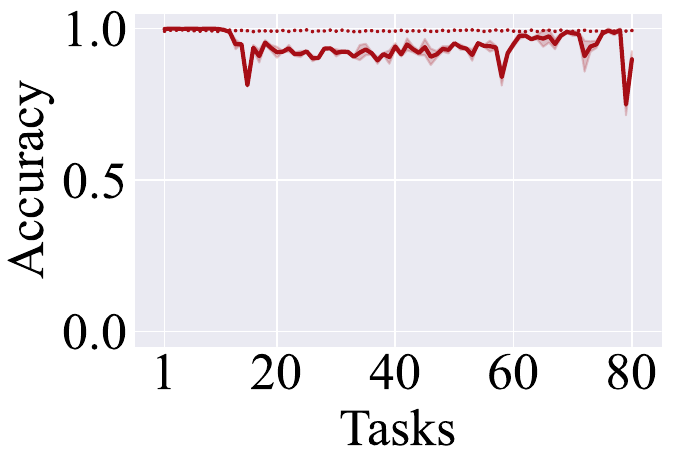}}
    \subfloat[NegGrad+]{\includegraphics[width=0.2\linewidth]{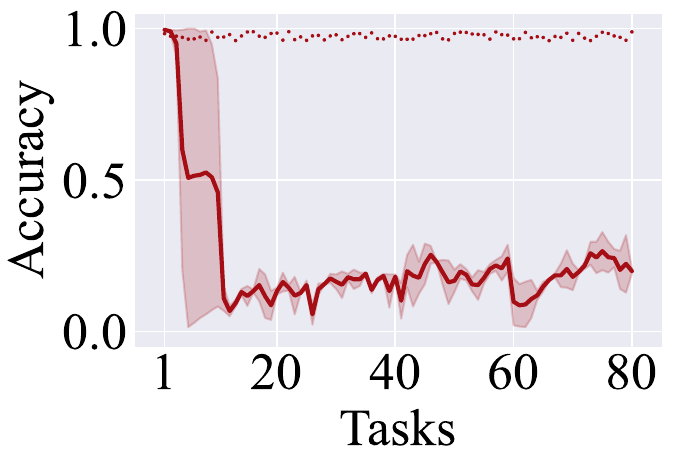}}
    \subfloat[RL]{\includegraphics[width=0.2\linewidth]{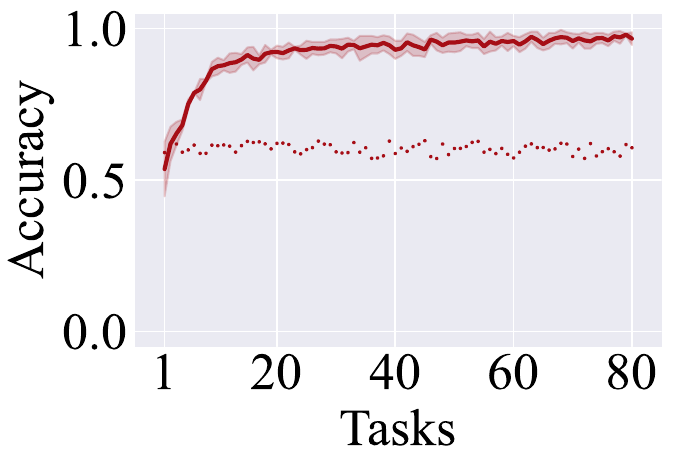}}
    \subfloat[SalUn]{\includegraphics[width=0.2\linewidth]{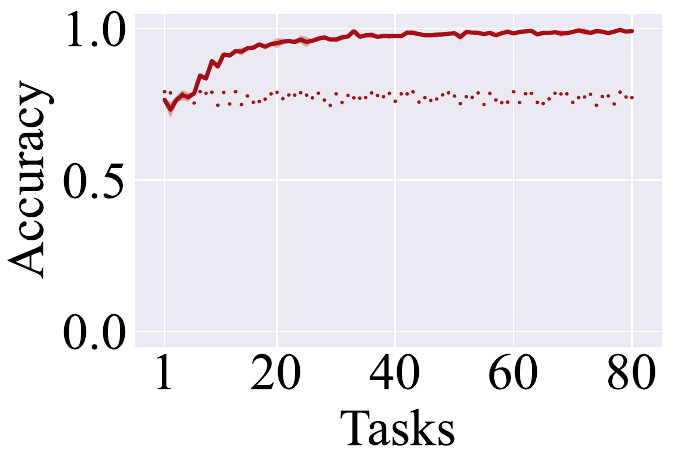}}
    \subfloat[MUNBa]{\includegraphics[width=0.2\linewidth]{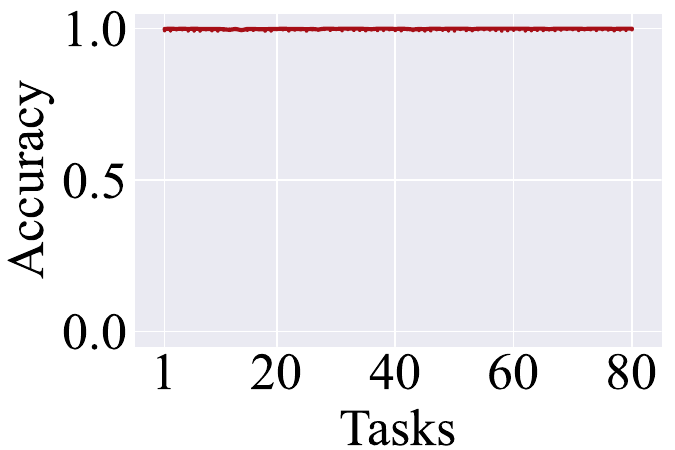}}

\vspace{-2mm}
\caption{Forgetting accuracy on CIFAR-100 with PreResNet-110 under the random data forgetting scenario.}
\label{fig:forget_acc_resnet_random}
\end{center}
\vspace{-4mm}
\end{figure*}

\begin{figure*}[!t]
\begin{center}

    \subfloat[FT]{\includegraphics[width=0.2\linewidth]{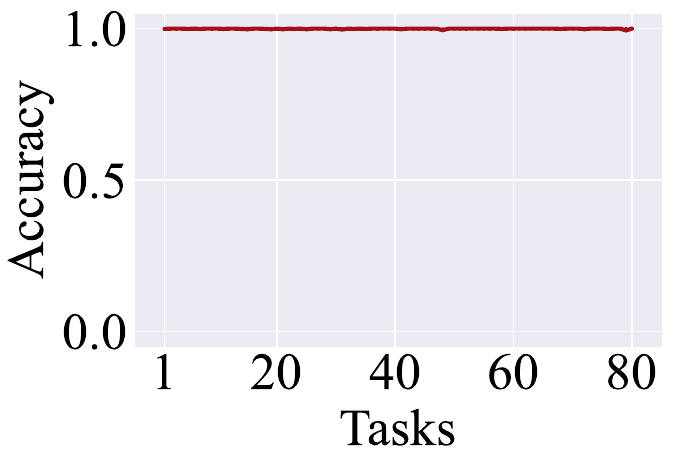}}
    \subfloat[NegGrad+]{\includegraphics[width=0.2\linewidth]{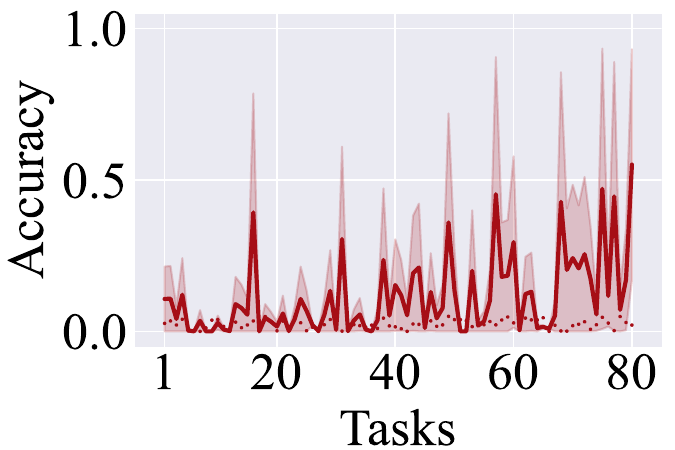}}
    \subfloat[RL]{\includegraphics[width=0.2\linewidth]{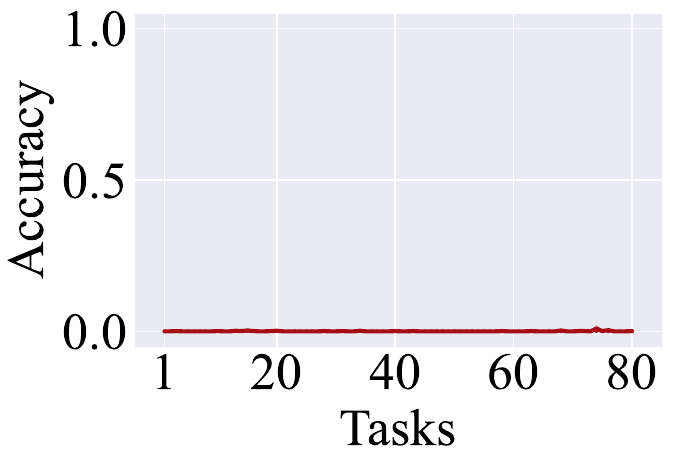}}
    \subfloat[SalUn]{\includegraphics[width=0.2\linewidth]{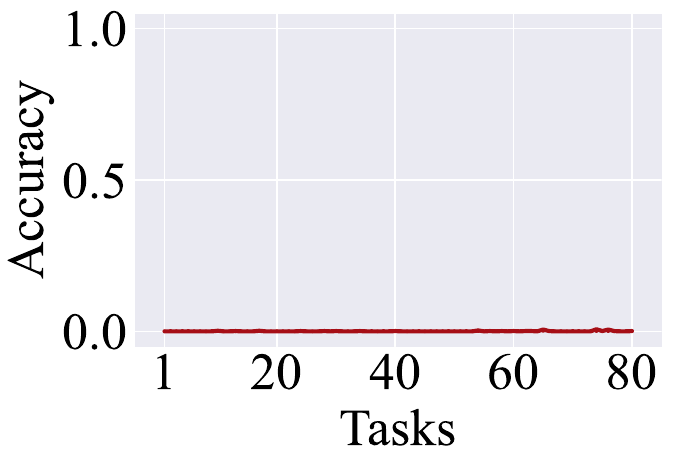}}
    \subfloat[MUNBa]{\includegraphics[width=0.2\linewidth]{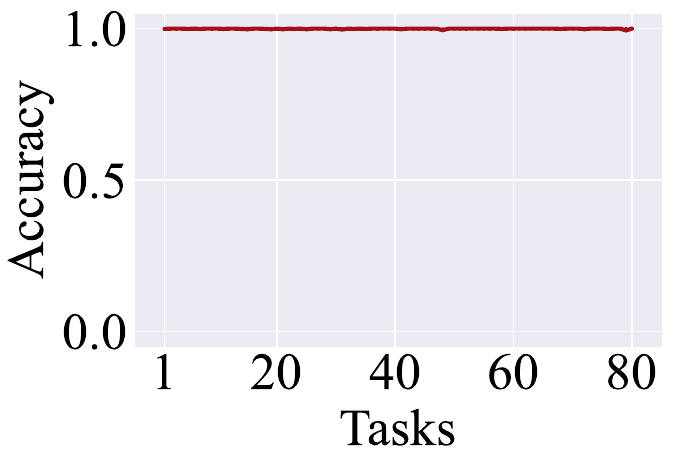}\label{fig:forget_acc_resnet_class_e}}

\vspace{-2mm}
\caption{Forgetting accuracy on CIFAR-100 with PreResNet-110 under the class-wise forgetting scenario.}
\label{fig:forget_acc_resnet_class}
\end{center}
\vspace{-6mm}
\end{figure*}
\begin{figure*}[!t]
\begin{center}

    \subfloat[FT]{\includegraphics[width=0.2\linewidth]{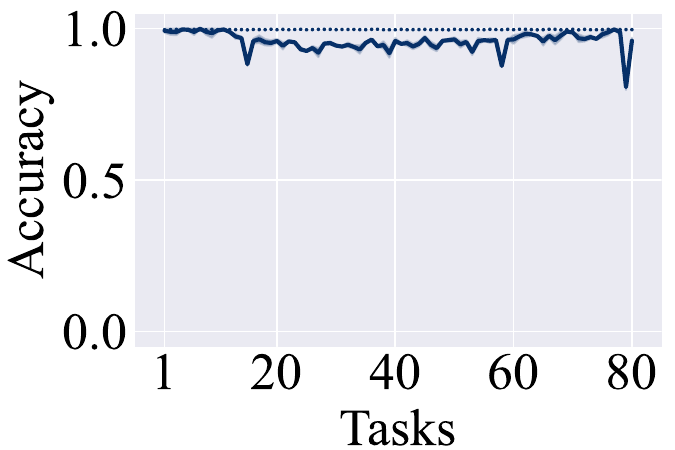}}
    \subfloat[NegGrad+]{\includegraphics[width=0.2\linewidth]{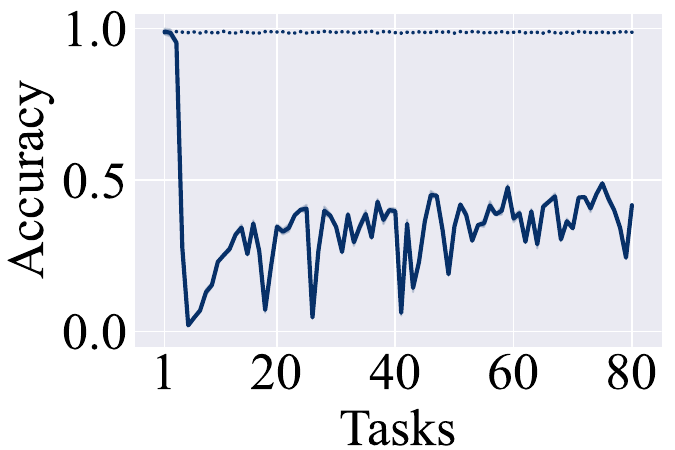}}
    \subfloat[RL]{\includegraphics[width=0.2\linewidth]{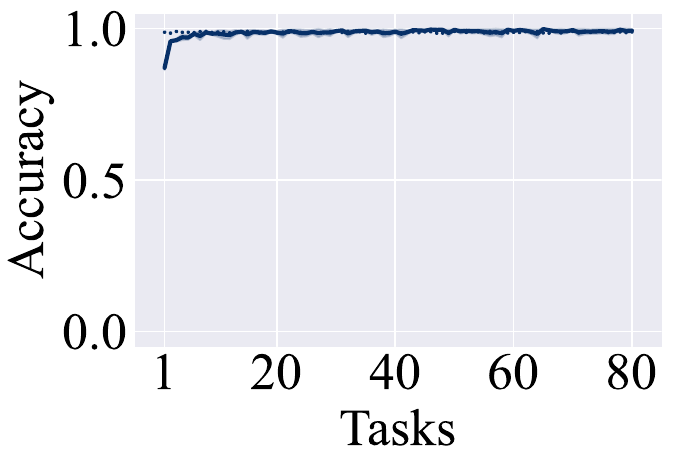}}
    \subfloat[SalUn]{\includegraphics[width=0.2\linewidth]{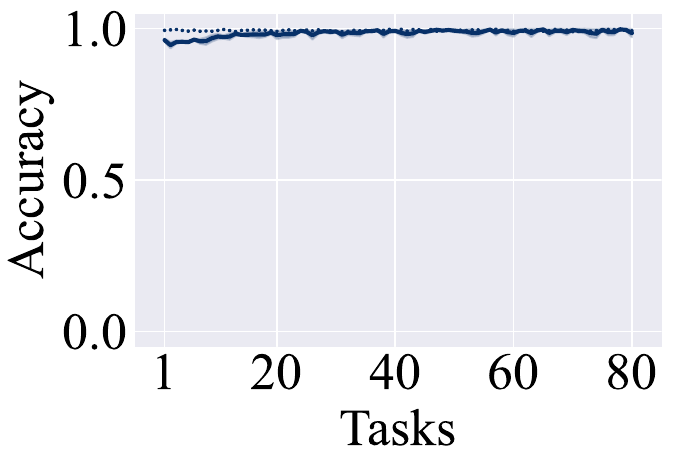}}
    \subfloat[MUNBa]{\includegraphics[width=0.2\linewidth]{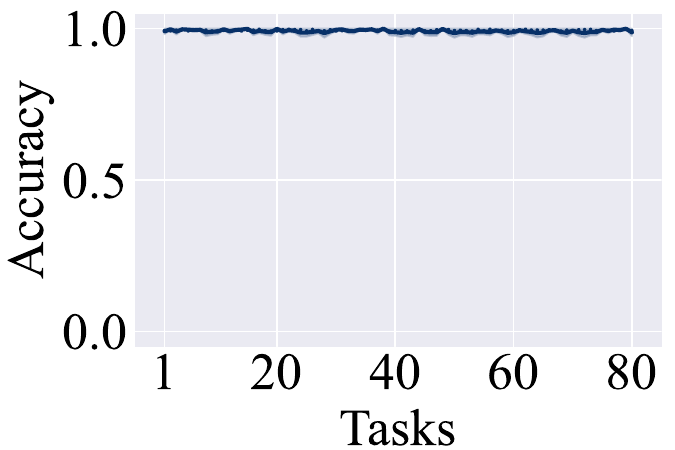}}

\vspace{-2mm}
\caption{Retain accuracy on CIFAR-100 with PreResNet-110 under the random data forgetting scenario.}
\label{fig:retain_acc_resnet_random}
\end{center}
\vspace{-6mm}
\end{figure*}

\begin{figure*}[!t]
\begin{center}

    \subfloat[FT]{\includegraphics[width=0.2\linewidth]{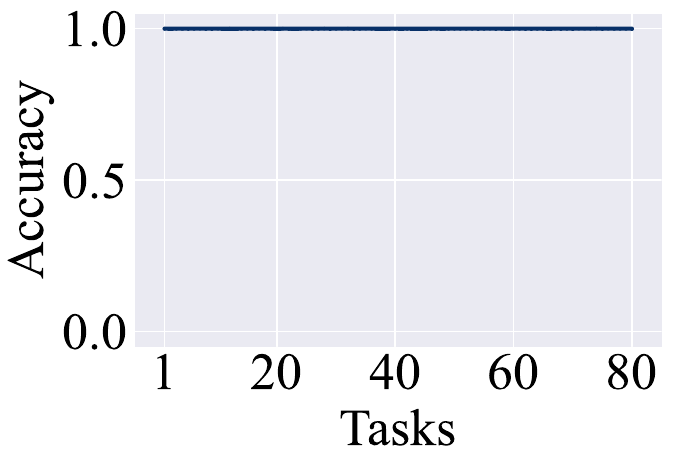}}
    \subfloat[NegGrad+]{\includegraphics[width=0.2\linewidth]{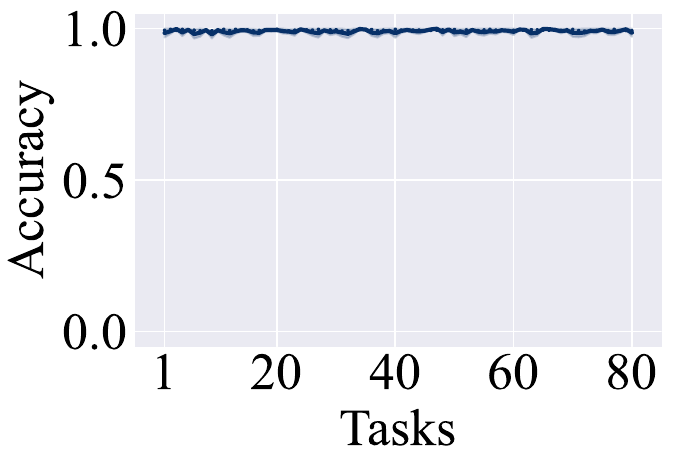}}
    \subfloat[RL]{\includegraphics[width=0.2\linewidth]{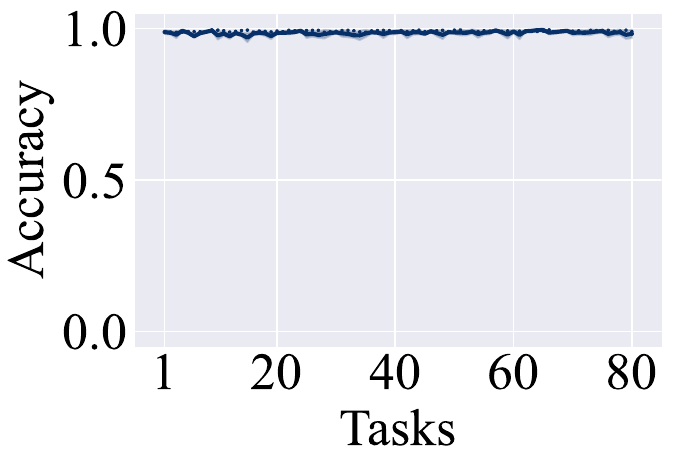}}
    \subfloat[SalUn]{\includegraphics[width=0.2\linewidth]{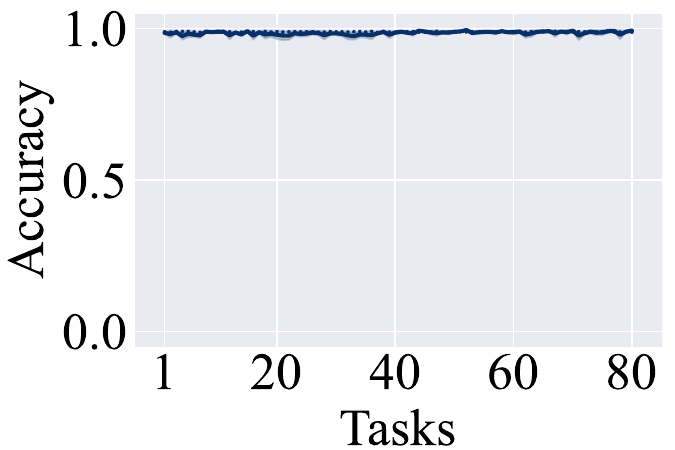}}
    \subfloat[MUNBa]{\includegraphics[width=0.2\linewidth]{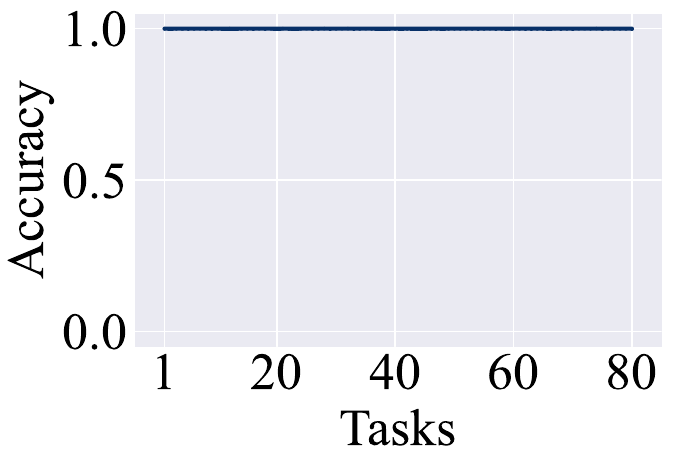}}

\vspace{-2mm}
\caption{Retain accuracy on CIFAR-100 with PreResNet-110 under the class-wise forgetting scenario.}
\label{fig:retain_acc_resnet_class}
\end{center}
\vspace{-6mm}
\end{figure*}
\begin{figure*}[!t]
\begin{center}

    \subfloat[FT]{\includegraphics[width=0.2\linewidth]{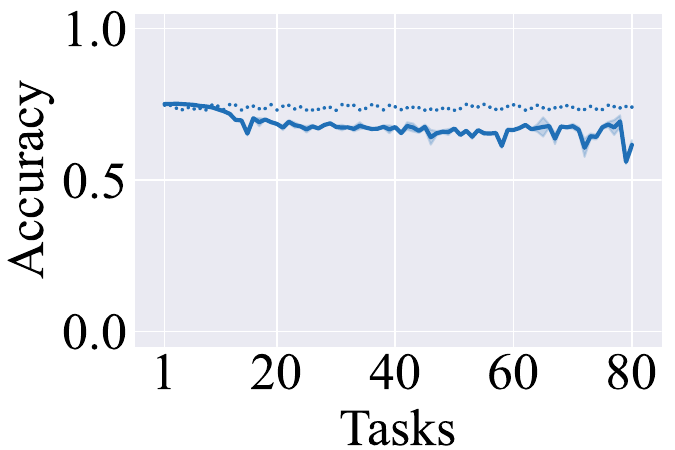}}
    \subfloat[NegGrad+]{\includegraphics[width=0.2\linewidth]{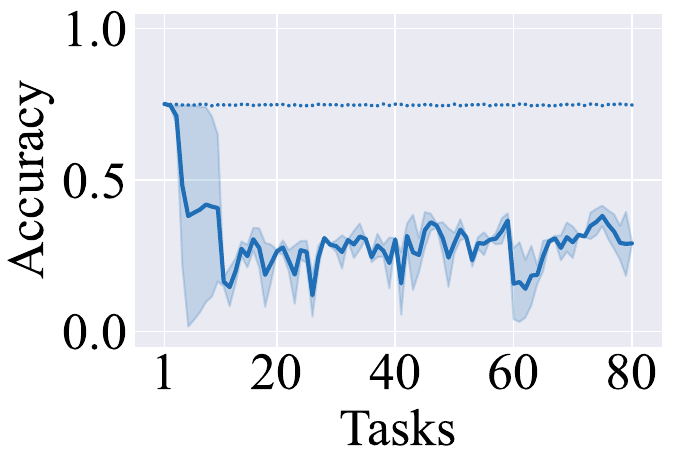}}
    \subfloat[RL]{\includegraphics[width=0.2\linewidth]{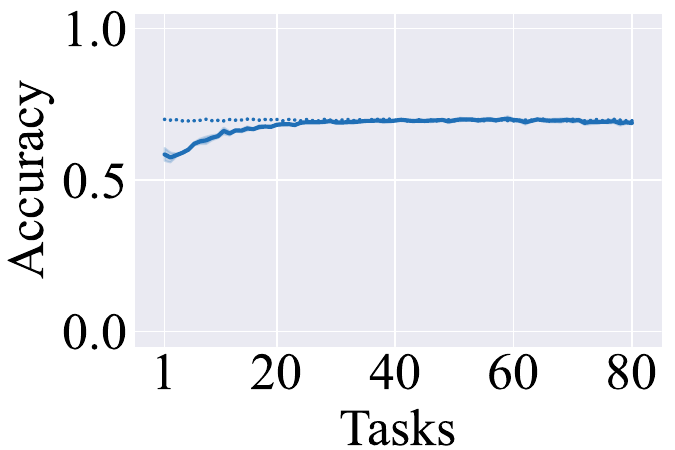}}
    \subfloat[SalUn]{\includegraphics[width=0.2\linewidth]{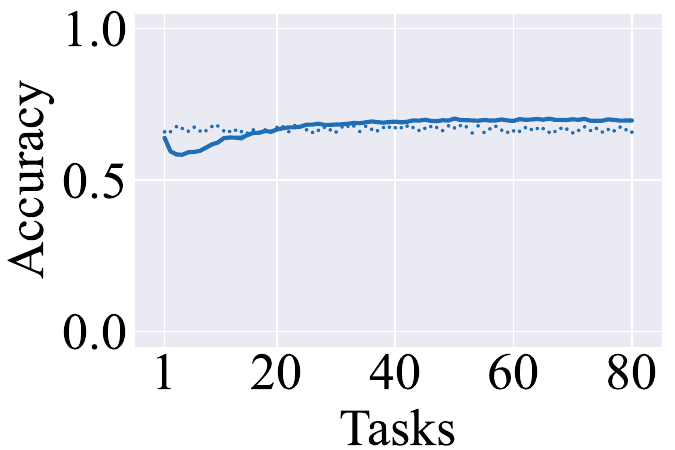}}
    \subfloat[MUNBa]{\includegraphics[width=0.2\linewidth]{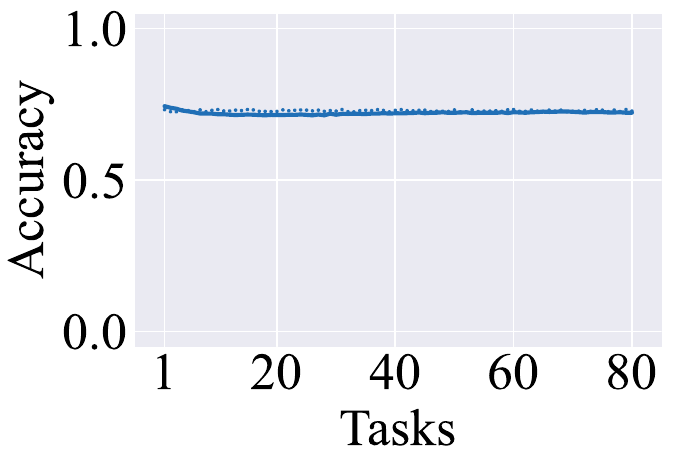}}

\vspace{-2mm}
\caption{Test accuracy on CIFAR-100 with PreResNet-110 under the random data forgetting scenario.}
\label{fig:test_acc_resnet_random}
\end{center}
\vspace{-6mm}
\end{figure*}

\begin{figure*}[!t]
\begin{center}

    \subfloat[FT]{\includegraphics[width=0.2\linewidth]{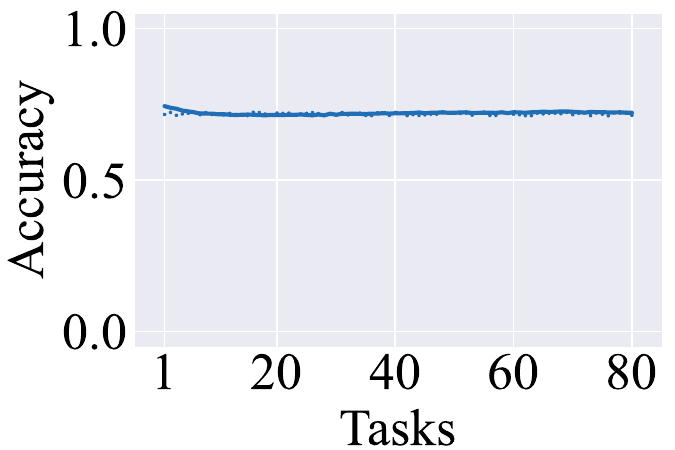}}
    \subfloat[NegGrad+]{\includegraphics[width=0.2\linewidth]{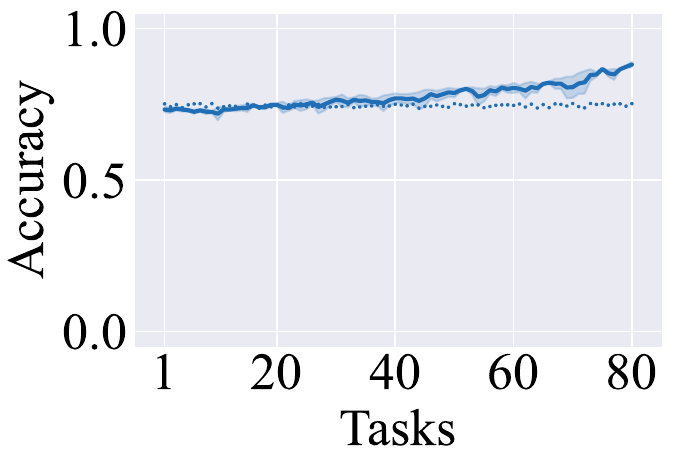}}
    \subfloat[RL]{\includegraphics[width=0.2\linewidth]{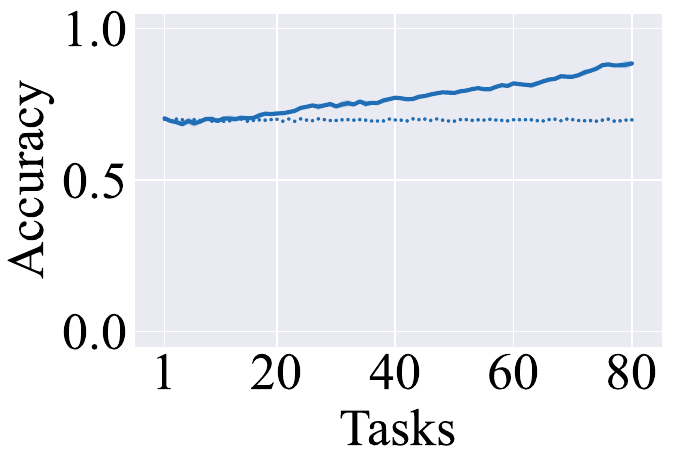}}
    \subfloat[SalUn]{\includegraphics[width=0.2\linewidth]{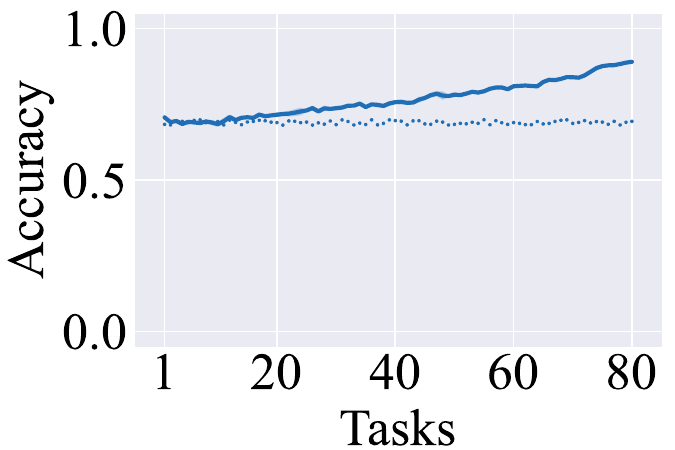}}
    \subfloat[MUNBa]{\includegraphics[width=0.2\linewidth]{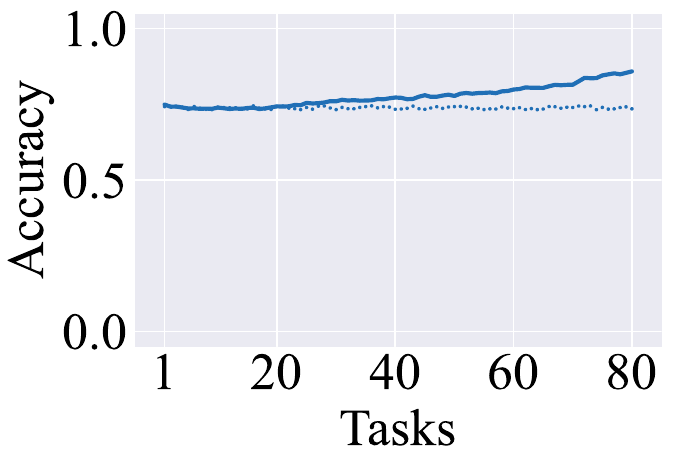}}

\vspace{-2mm}
\caption{Test accuracy on CIFAR-100 with PreResNet-110 under class-wise forgetting scenario.}
\label{fig:test_acc_resnet_class}
\end{center}
\vspace{-6mm}
\end{figure*}

\begin{figure*}
\begin{center}

    \subfloat[FT]{\includegraphics[width=0.2\linewidth]{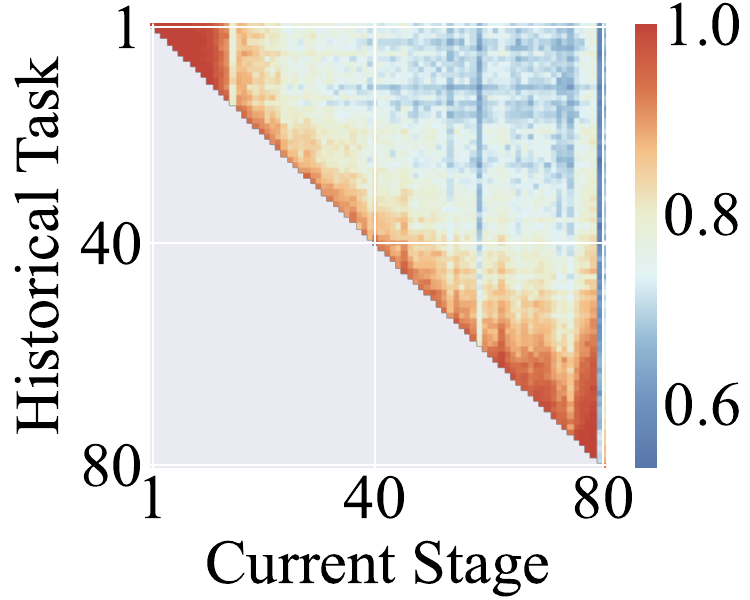}
    \label{fig:heatmap_random_resnet_a}}
    \subfloat[NegGrad+]{\includegraphics[width=0.2\linewidth]{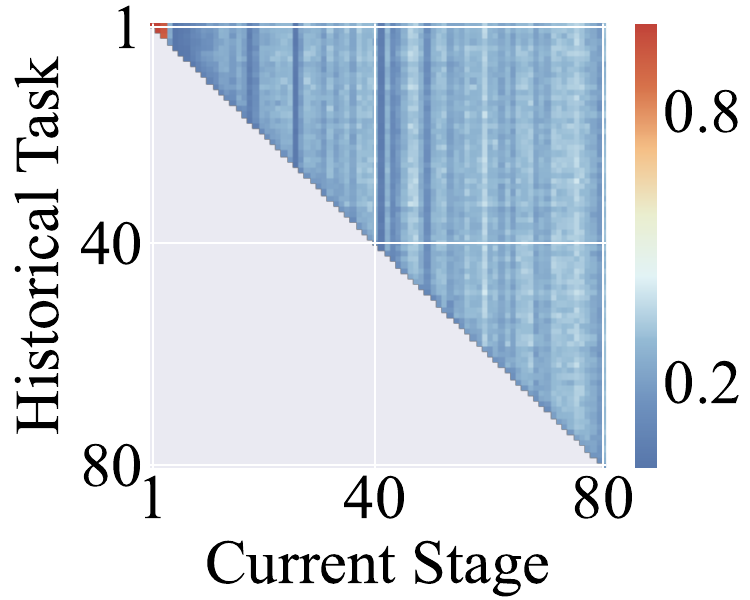}}
    \subfloat[RL]{\includegraphics[width=0.2\linewidth]{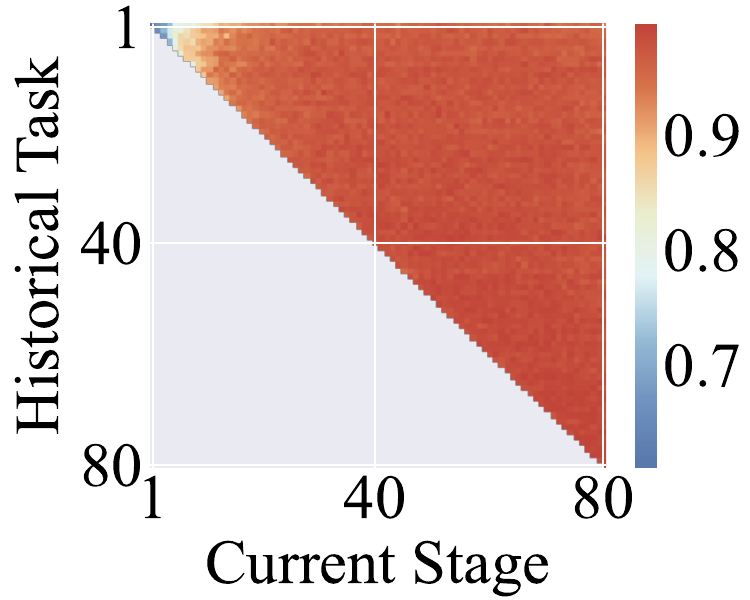}}
    \subfloat[SalUn]{\includegraphics[width=0.2\linewidth]{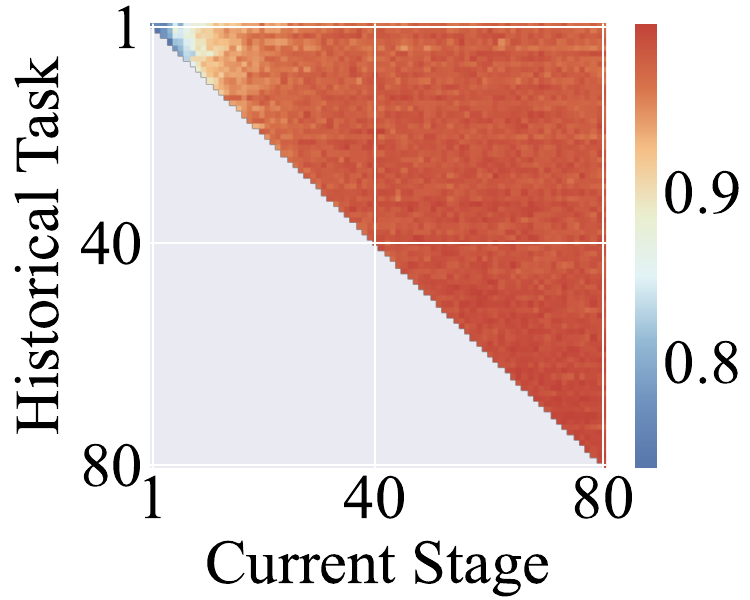}}
    \subfloat[MUNBa]{\includegraphics[width=0.2\linewidth]{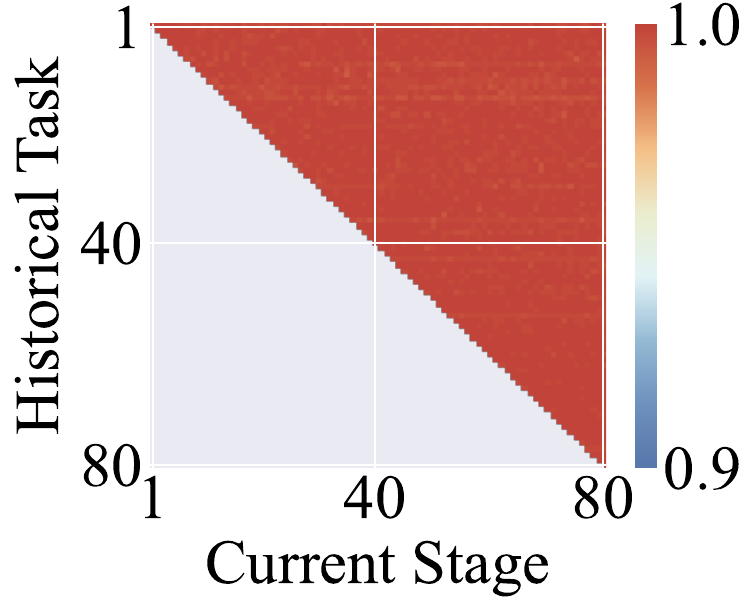}
    \label{fig:heatmap_random_resnet_e}}

\vspace{-2mm}
\caption{\footnotesize{Forgetting accuracy heatmaps on CIFAR-100 with PreResNet-110 for historical tasks on the random data forgetting scenario.}}
\label{fig:heatmap_random_resnet}
\end{center}
\vspace{-4mm}
\end{figure*}
\begin{figure*}[!t]
\begin{center}

    \subfloat[FT]{\includegraphics[width=0.2\linewidth]{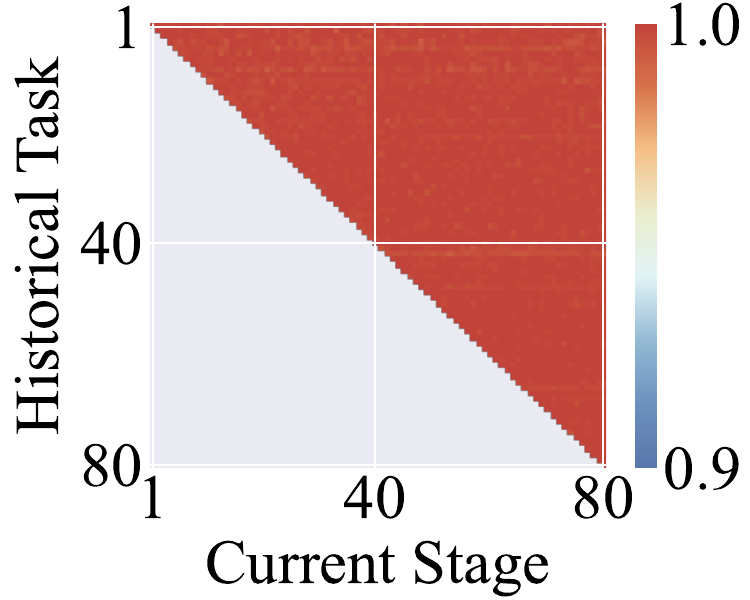}\label{fig:heatmap_class_resnet_a}}
    \subfloat[NegGrad+]{\includegraphics[width=0.2\linewidth]{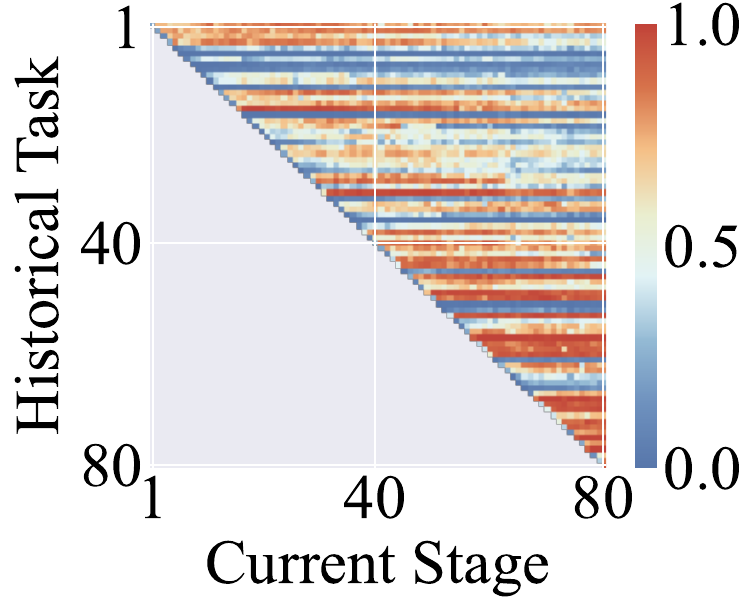}\label{fig:heatmap_class_resnet_b}}
    \subfloat[RL]{\includegraphics[width=0.2\linewidth]{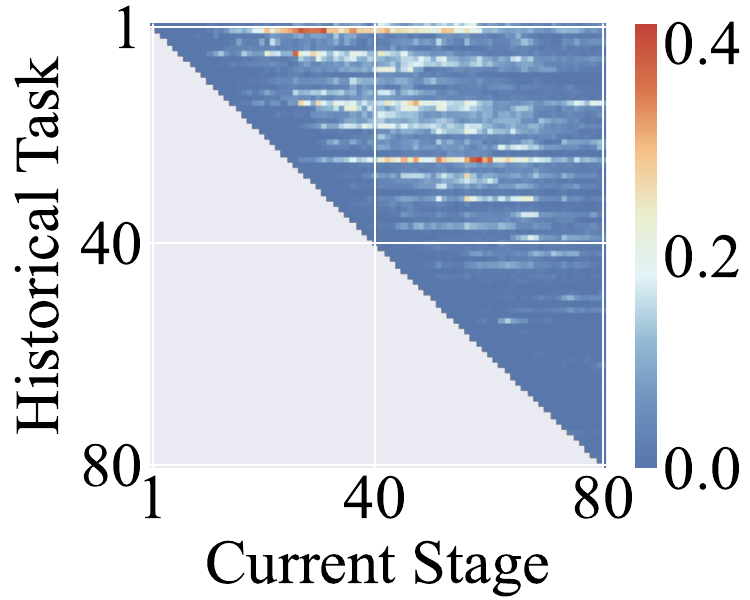}\label{fig:heatmap_class_resnet_c}}
    \subfloat[SalUn]{\includegraphics[width=0.2\linewidth]{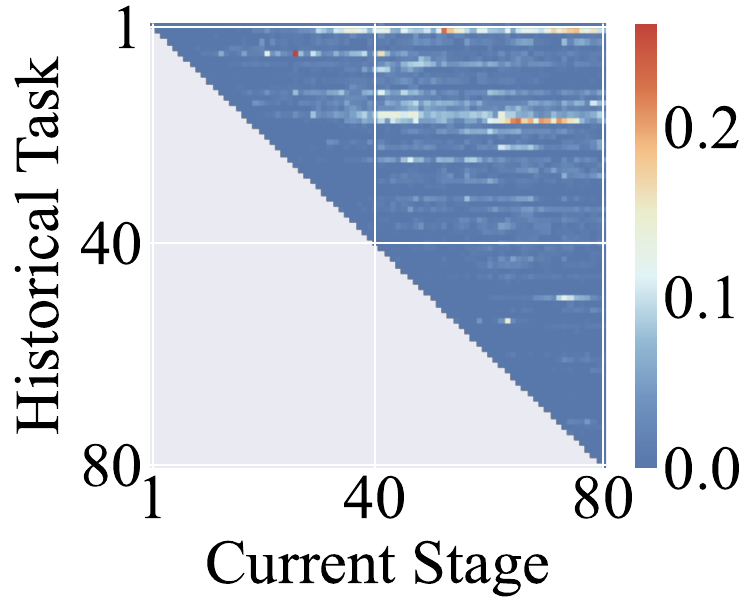}\label{fig:heatmap_class_resnet_d}}
    \subfloat[MUNBa]{\includegraphics[width=0.2\linewidth]{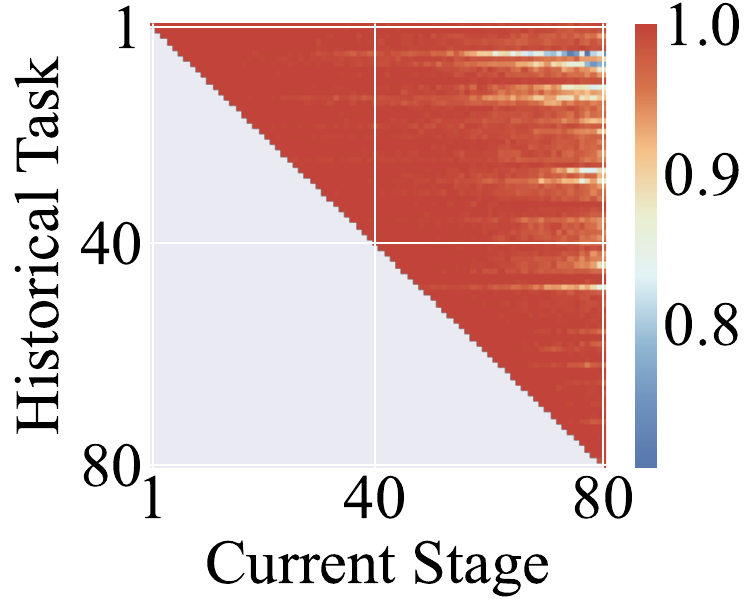}\label{fig:heatmap_class_resnet_e}}

\vspace{-2mm}
\caption{\footnotesize{Forgetting accuracy heatmaps on CIFAR-100 with PreResNet-110 for historical tasks on the class-wise forgetting scenario.}}
\label{fig:heatmap_class_resnet}
\end{center}
\end{figure*}

\begin{figure*}[t]
\begin{center}

    \subfloat[RL]{\includegraphics[width=0.2\linewidth]{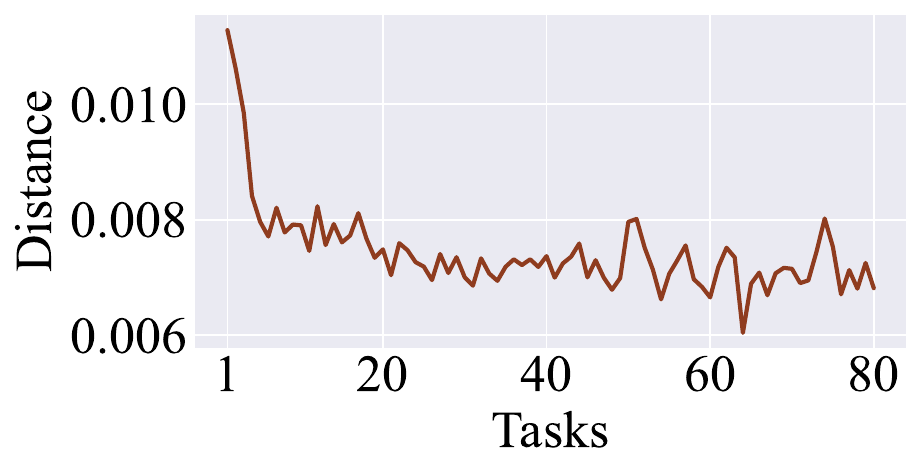}}
    \subfloat[SalUn]{\includegraphics[width=0.2\linewidth]{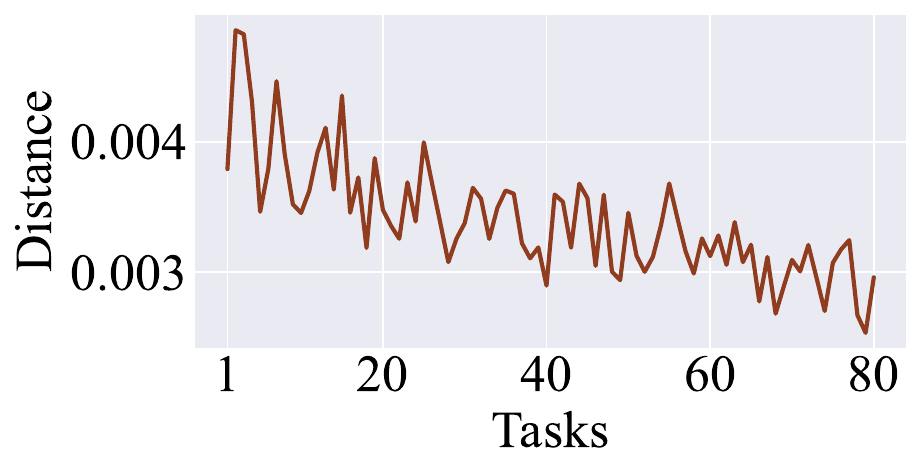}}
    \subfloat[MUNBa]{\includegraphics[width=0.2\linewidth]{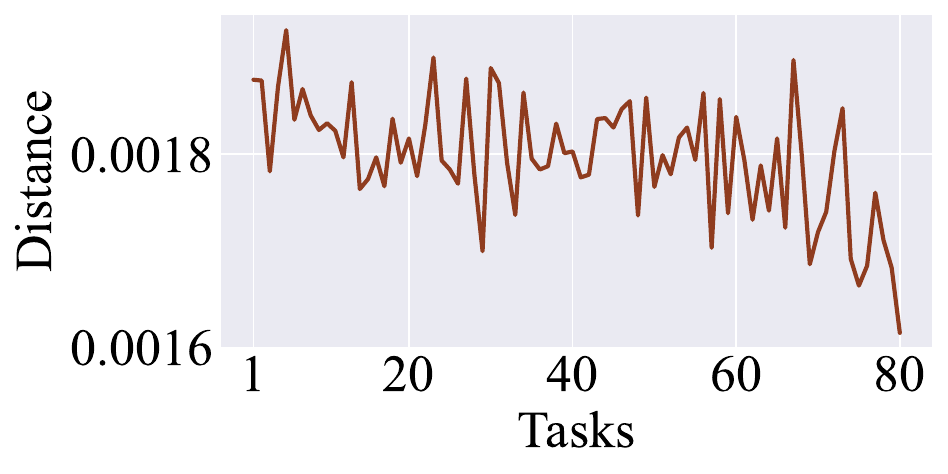}}

\vspace{-2mm}
\caption{\footnotesize{L2 norm of parameter differences between consecutive updates on CIFAR-100 with PreResNet-110 under the random data forgetting scenario.}}
\label{fig:l2_resnet_random}
\end{center}
\vspace{-6mm}
\end{figure*}

\begin{figure*}[t]
\begin{center}

    \subfloat[NegGrad+]{\includegraphics[width=0.2\linewidth]{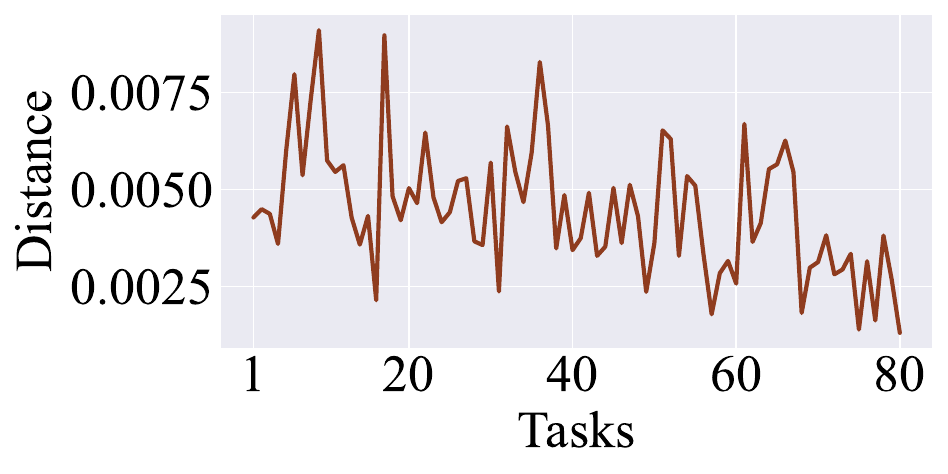}}
    \subfloat[RL]{\includegraphics[width=0.2\linewidth]{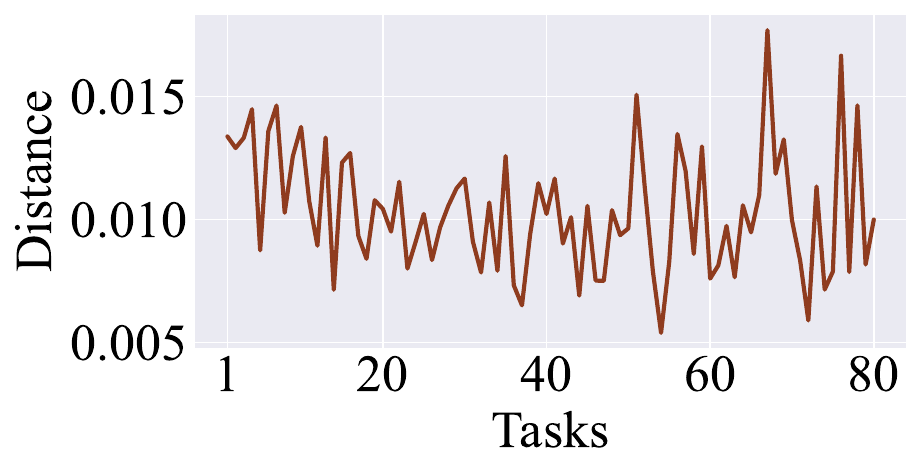}\label{fig:l2_resnet_class_c}}
    \subfloat[SalUn]{\includegraphics[width=0.2\linewidth]{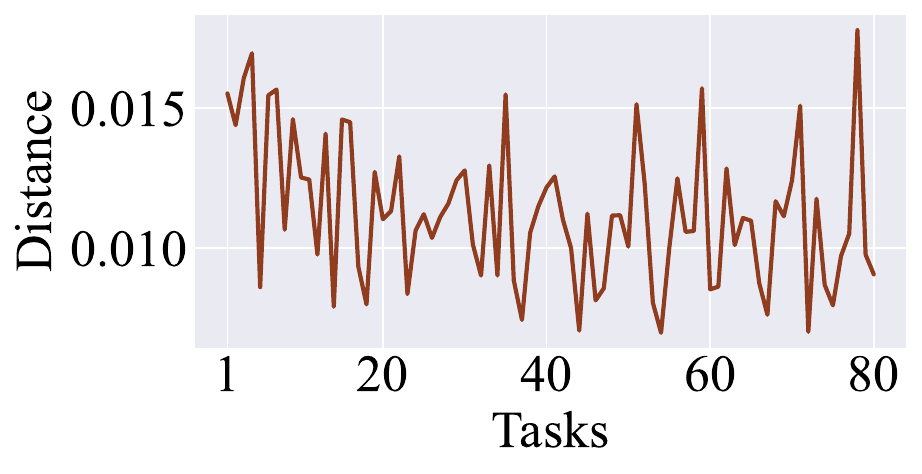}\label{fig:l2_resnet_class_d}}

\vspace{-2mm}
\caption{\footnotesize{L2 norm of parameter differences between consecutive updates on CIFAR-100 with PreResNet-110 under the class-wise forgetting scenario.}}
\label{fig:l2_resnet_class}
\end{center}
\vspace{-6mm}
\end{figure*}
\subsection{Diminishing Update Magnitude}
In this section, we analyze the update magnitude for the experimental cases that satisfy the preconditions of our theoretical framework. Theorems~\ref{thm:operator_growth} and~\ref{thm:expanding_mode_general} predict that as $W$ approaches saturation, the suppression mechanism increasingly drives parameter updates away from $W$. This phenomenon is expected to manifest as a systematic reduction in the update magnitude, defined as $\mathbf{u}_t := \bm{\theta}_{t} - \bm{\theta}_{t-1}$. Our empirical results on the cases with forward failure consistently exhibit this diminishing update magnitude as shown in Figs.~\ref{fig:l2_resnet_random}-\ref{fig:l2_resnet_class}. These observations provide direct evidence that the model becomes increasingly "frozen" in the parameter space, progressively losing its capacity to generate the meaningful updates required to fulfill successive forget requests.

\begin{figure*}[!t]
\begin{center}
    \subfloat[RL]{\includegraphics[width=0.2\linewidth]{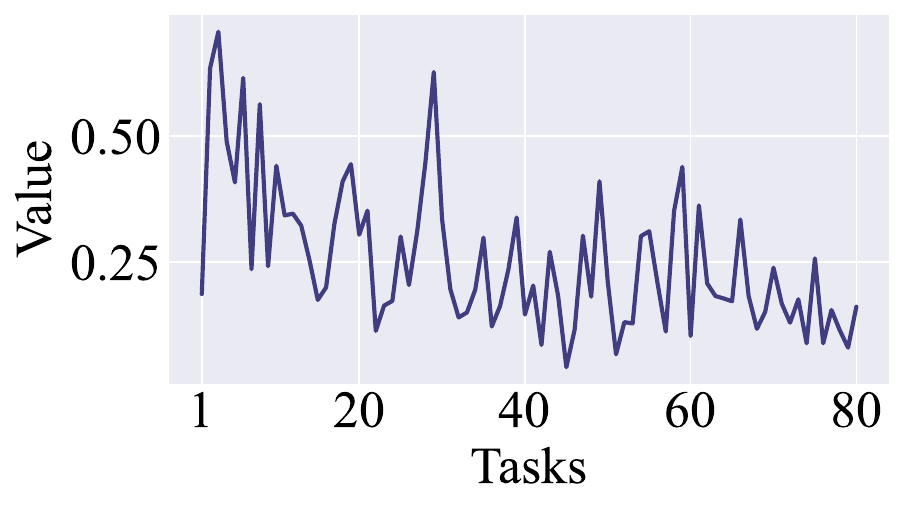}\label{fig:resnet_random_energy_a}}
    \subfloat[SalUn]{\includegraphics[width=0.2\linewidth]{figures/subspace/resnet/random_1_salun_1_ER.pdf}}
    \subfloat[MUNBa]{\includegraphics[width=0.2\linewidth]{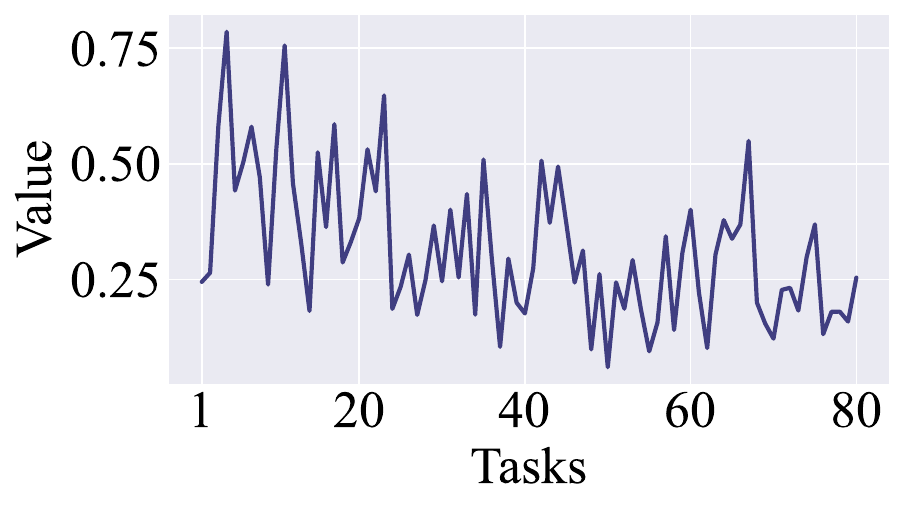}}

\vspace{-2mm}
\caption{Energy Ratio for forward failure diagnosis on CIFAR-100 with PreResNet-110 under the random data forgetting scenario.}
\label{fig:resnet_random_energy}
\end{center}
\vspace{-6mm}
\end{figure*}
\begin{figure*}[!t]
\begin{center}
    \subfloat[NegGrad+]{\includegraphics[width=0.2\linewidth]{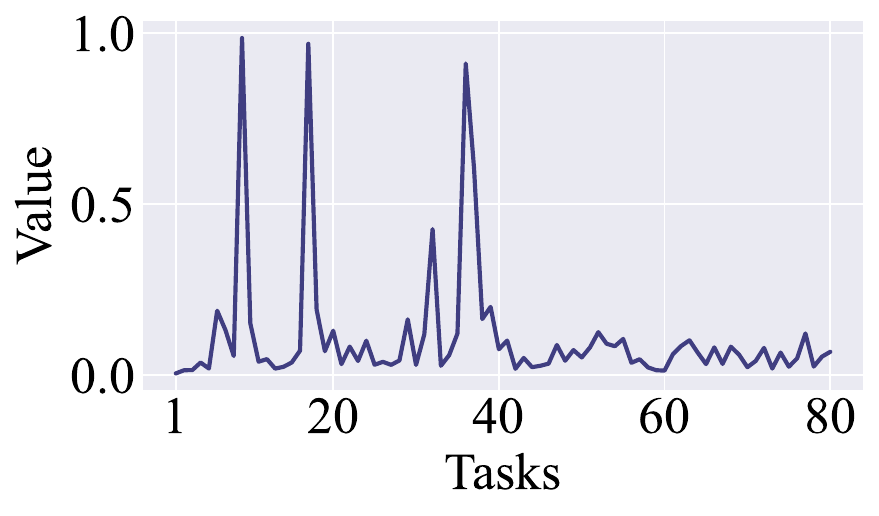}\label{fig:resnet_class_energy_a}}
    \subfloat[RL]{\includegraphics[width=0.2\linewidth]{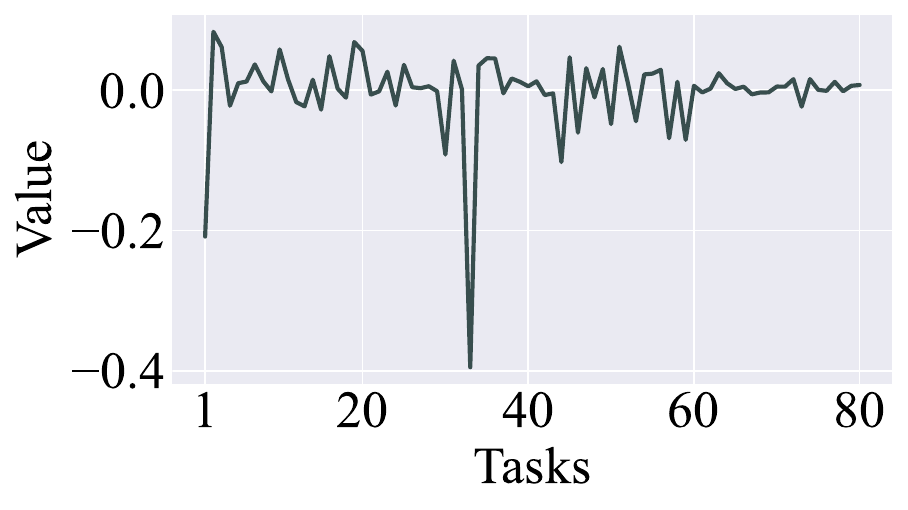}\label{fig:resnet_class_energy_b}}
    \subfloat[SalUn]{\includegraphics[width=0.2\linewidth]{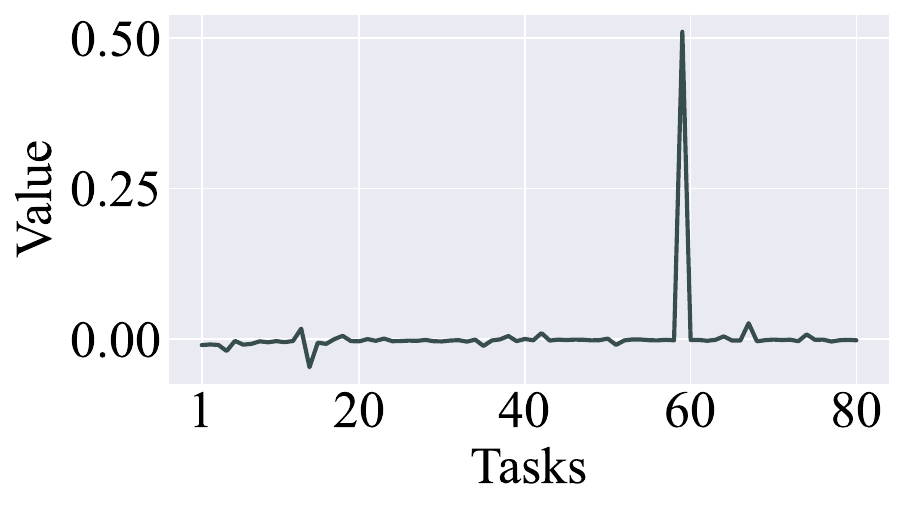}\label{fig:resnet_class_energy_c}}

\vspace{-2mm}
\caption{Energy Ratio or Coefficient for forward failure diagnosis on CIFAR-100 with PreResNet-110 under the class-wise forgetting scenario.}
\label{fig:resnet_class_energy}
\end{center}
\vspace{-8mm}
\end{figure*}
\subsection{Diagnostic Quantities for Plasticity Collapse (ER and CO)}

We now apply the diagnostic metrics $\mathrm{ER}_t$ and $\mathrm{CO}_t$ introduced in Section~\ref{subsec:metrics} to empirically validate our theoretical predictions. The shared subspace $W$ is estimated from the sequence of updates $\{\mathbf{u}_1, \dots, \mathbf{u}_T\}$ by performing PCA on the stacked update matrix $\mathbf{U}$. For visualization clarity, we set the subspace rank to $r = 1$, using the leading right singular vector; we note that the observed patterns remain consistent for $r > 1$. Figs~\ref{fig:resnet_random_energy_a}-\ref{fig:resnet_class_energy_a} illustrate the evolution of $\mathrm{ER}_t$ across sequential stages. In scenarios exhibiting forward failure, $\mathrm{ER}_t$ decays progressively toward zero across all evaluated methods. This trend confirms the predicted suppression mechanism: as the accumulation of $\mathbf{A}_{t:s}$ induces exponential growth along $W$, the algorithm is forced to route updates away from $W$ to maintain model stability. Notably, the decay in $\mathrm{ER}_t$ closely tracks the degradation of forgetting accuracy, providing quantitative evidence that the geometric saturation of $W$ is the proximate cause of forward failure. Finally, Figs~\ref{fig:resnet_class_energy_b}-\ref{fig:resnet_class_energy_c} present the $\mathrm{CO}_t$ results for cases characterized by backward failure.

\end{document}